\documentclass[10pt,twocolumn,letterpaper]{article}

\usepackage[pagenumbers]{iccv} %
\usepackage{tipa}
\usepackage{todonotes}
\usepackage{multirow}
\usepackage{textcomp}
\usepackage{colortbl}
\usepackage{mathtools}
\usepackage{amssymb}  %
\usepackage{amsfonts}
\usepackage{amsthm}
\usepackage{tabu}
\usepackage{siunitx}
\usepackage{enumitem}

\newtheorem{theorem}{Theorem}[section]

\theoremstyle{definition}
\newtheorem{definition}[theorem]{Definition}

\newcommand{\first}[1]{\cellcolor{green!100}{#1}}

\newcommand{\ours}[0]{DaD}

\definecolor{iccvblue}{rgb}{0.21,0.49,0.74}
\usepackage[pagebackref,breaklinks,colorlinks,citecolor=iccvblue]{hyperref}

\title{\ours: Distilled Reinforcement Learning for Diverse Keypoint Detection}

\author{Johan Edstedt$^1$
\quad
Georg Bökman$^2$
\quad
 Mårten Wadenbäck$^1$
\quad
 Michael Felsberg$^1$
 \\
{\normalsize $^1$Linköping University, $^2$Chalmers University of Technology}
}

\begin{document}
\twocolumn[{%
\centering
\renewcommand\twocolumn[1][]{#1}%
\maketitle
    \includegraphics[width=\linewidth]{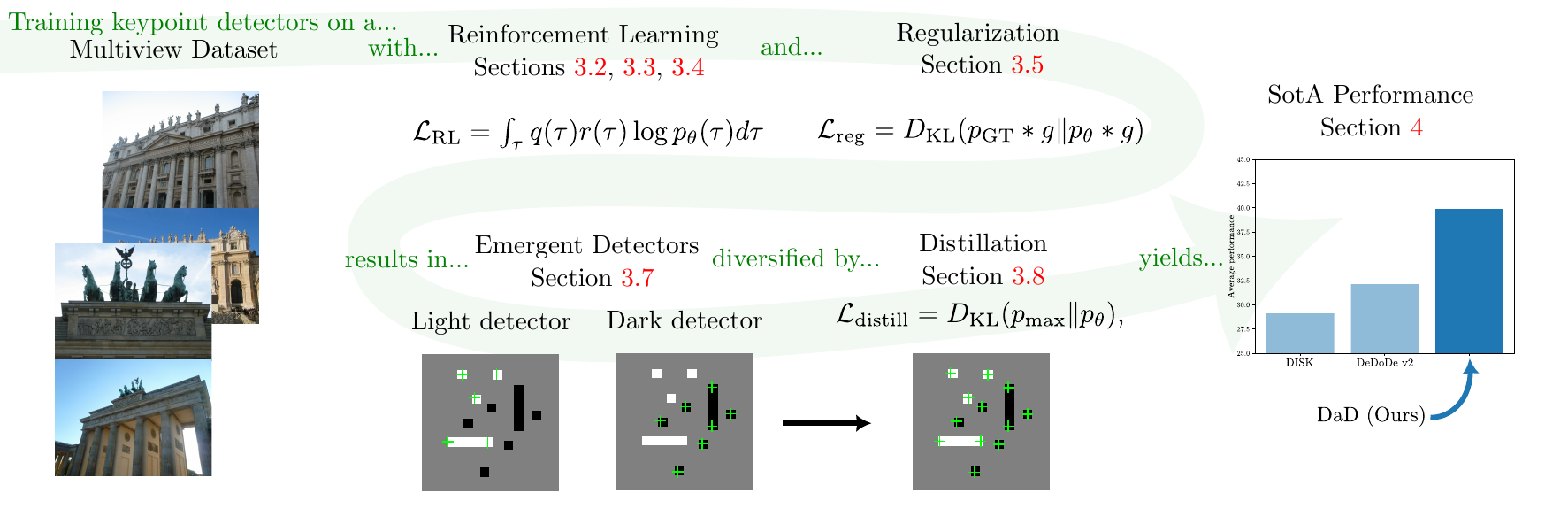}
    \captionof{figure}{
    \textbf{Overview of \ours.} 
    We present a method to train a keypoint detector that requires neither a descriptor, nor supervision from SfM tracks, yet achieves SotA performance.
    We use reinforcement learning (\Cref{sec:rl,sec:reward,sec:sampling}) to iteratively improve our detector through a two-view repeatability reward in combination with a simple regularization objective (\Cref{sec:regularization}). We find that two types of detectors, which detect only light and dark keypoints respectively, emerge from optimizing the RL objective (\Cref{sec:emerge}). This is problematic, as many repeatable keypoints are missed. We tackle this by combining the detectors through point-wise maximum knowledge distillation (\Cref{sec:distill}) to a final powerful and diverse keypoint detector, which we call~\ours. \ours~sets a new state-of-the-art for keypoint detection, as our experiments in \Cref{sec:results} show.}
    \label{fig:overview}}]
\maketitle
\begin{abstract}
Keypoints are what enable Structure-from-Motion (SfM) systems to scale to thousands of images.
However, designing a keypoint detection objective is a non-trivial task, as SfM is non-differentiable.
Typically, an auxiliary objective involving a descriptor is optimized. 
This however induces a dependency on the descriptor, which is undesirable.
In this paper we propose a fully self-supervised and descriptor-free objective for keypoint detection, through reinforcement learning.
To ensure training does not degenerate, we leverage a balanced top-K sampling strategy.
While this already produces competitive models, we find that two qualitatively different types of detectors emerge, which are only able to detect light and dark keypoints respectively.
To remedy this, we train a third detector, \ours, that optimizes the Kullback--Leibler divergence of the pointwise maximum of both light and dark detectors.
Our approach significantly improve upon SotA across a range of benchmarks.
Code and model weights are publicly available at \url{https://github.com/parskatt/dad}.
\vspace{-3em}
\vfill
\end{abstract}
    
\section{Introduction}
\label{sec:intro}
\begin{figure*}
    \centering
    \includegraphics[width=\linewidth]{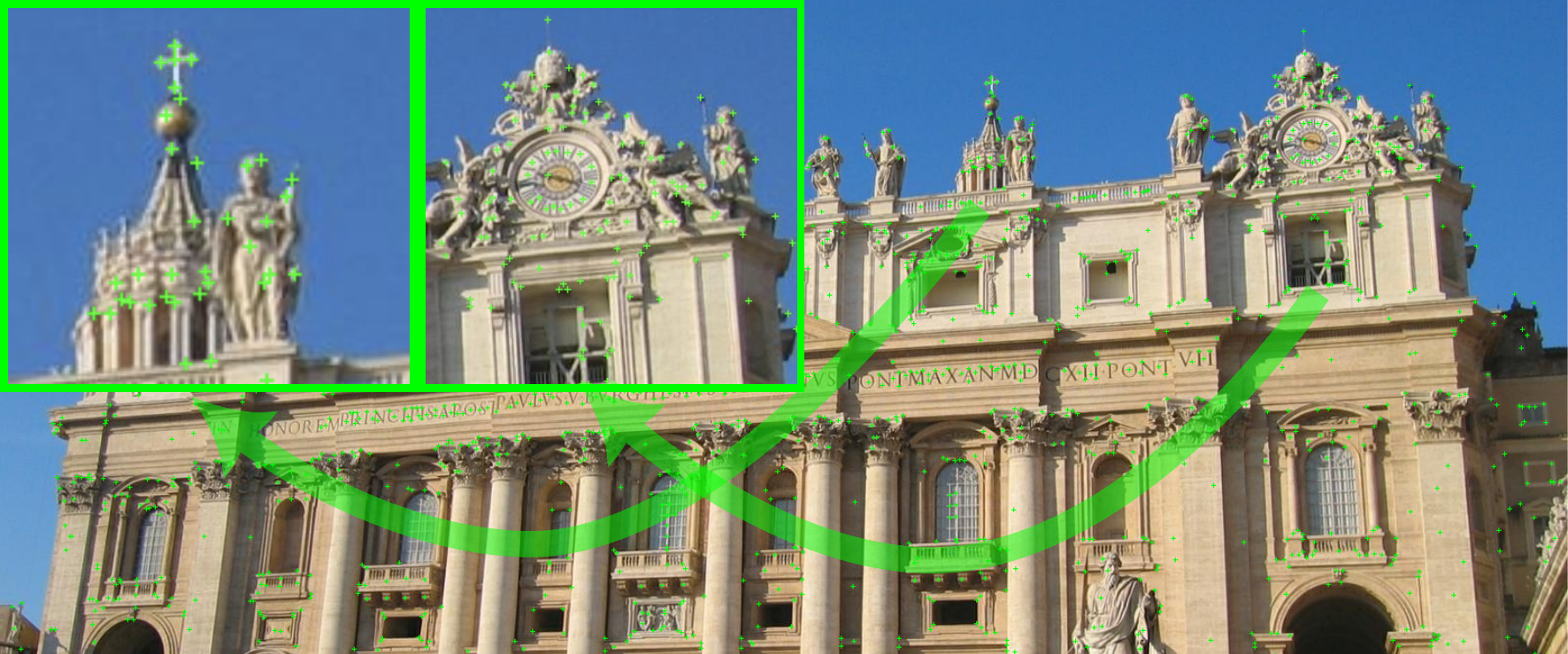}
    \caption{
    \textbf{Qualitative example of \ours~keypoint detections.} We find a fundamental issue with previous rotation invariant self-supervised detectors which only lets them detect \emph{light} or \emph{dark} types of keypoints (see \Cref{sec:emerge} for details). We remedy this through point-wise maximum knowledge distillation (see \Cref{sec:distill}). As can be seen in the figure, our approach has no such issue, see e.g., the light keypoints on the cross (left zoom in), and the dark keypoints on the building edge (right zoom in). A qualitative comparison with previous self-supervised detectors (where this issue occurs) is presented in \Cref{suppl:qualitative-comparison}.}
    \label{fig:qualitative}
\end{figure*}
 Structure-from-Motion (SfM) systems~\citep{opensfm, theia, schoenberger2016sfm, pan2024glomap} require repeatable observations of certain 3D points, called \emph{keypoints}. Models $p_{\theta}$ detecting sets of keypoints $\mathcal{K} = \{\mathbf{x}_k\in\mathbb{R}^2\}^K_{k=1} \sim p_{\theta}(\mathbf{x}|\mathbf{I})$ are called detectors. Traditionally~\citep{lowe2004distinctive, detone2018superpoint} keypoints are described by a \emph{descriptor} $d_{\theta}$, whose \emph{feature vectors} $d_{\theta}(\mathcal{K}|\mathbf{I})\in \mathbb{R}^{K\times D}$ are matched, either by nearest neighbours~\citep{lowe2004distinctive}, or by neural networks~\citep{sarlin20superglue, lindenberger2023lightglue}. While this setup is convenient, it forces a dependency between the detector, the descriptor, and the matcher~\citep{edstedt2024dedode}, as the detector is typically trained to detect mutual nearest neighbours of the the descriptor~\citep{tyszkiewicz2020disk, zhao2022alike, Zhao2023ALIKED}, the descriptor shares weights with the detector~\citep{dusmanu2019d2}, and the matcher is conditioned on the descriptor~\citep{sarlin20superglue,lindenberger2023lightglue}.

As an alternative, there has been a push towards \emph{detector-free} matching. In this paradigm, dense feature vectors are computed, matched densely at a coarse resolution, and further refined at higher resolution, either sparsely~\citep{sun2021loftr, chen2022aspanformer, wang2024efficient} or densely~\citep{truong2023pdc, edstedt2023dkm, edstedt2024roma}. These approaches have significantly higher accuracy than their detector-descriptor counterparts. 
However, since they lack keypoints, running larger scale reconstruction, while performant~\citep{he2024detector}, becomes computationally cumbersome and adds additional complexity.

In the same vein, decoupling the detector from the descriptor has also gained popularity~\cite{cieslewski2019matching,cieslewski2019sips,barroso2019key,lee2022self,santellani2023strek,edstedt2024dedodev2}.
While the decoupled approach has several advantages~\citep{li2022decoupling}, it is not obvious how to formulate a fully self-supervised objective without implicitly using a descriptor.
In the state-of-the-art method DeDoDe v2~\citep{edstedt2024dedodev2}, a heuristic combination of a detection prior (from SfM tracks), and a self-supervised objective is used.
This is unsatisfactory in the sense that the detector still depends on the descriptor used in the SfM pipeline, thus biasing the detector.

A promising direction for self-supervised keypoint detection is through reinforcement learning~\citep{sutton2018reinforcement}.
ReinforcedFP~\citep{bhowmik2020reinforced} use a pose-based reward to finetune SuperPoint~\citep{detone2018superpoint}, however, the gradients from this reward are not sufficiently informative to train a model from scratch~\citep{tyszkiewicz2020disk}.
DISK~\citep{tyszkiewicz2020disk} instead optimize a repeatability objective, which yields significantly better gradients.
In order to be able to optimize an on-policy objective, they perform patch-wise sampling and rejection, which leads to artifacts near patch boundaries~\citep{santellani2023strek}.
S-TREK~\citep{santellani2023strek} instead propose to use an off-policy version of the policy gradient objective, with a sequential sampling procedure where they avoid repeated sampling by using a sampling avoidance radius during training. 
However, this sampling procedure is dropped at inference time and replaced with non-max-suppression and top-K, creating an alignment gap from training to inference.

In this paper, we train a self-supervised detector through reinforcement learning, while avoiding the above issues by a balanced top-K sampling strategy.
Intriguingly, on this quest, we discover a fundamental issue with rotation invariant reinforcement learning-like objectives which causes detectors to detect either only \emph{light} keypoints or \emph{dark} keypoints. To remedy this issue, we propose a point-wise maximum distillation objective, that, given a light detector and a dark detector, distills them into a diverse detector that detects both types of keypoints.
We call our resulting detector \textbf{\ours}.
An overview and qualitative example of \ours~is presented in~\Cref{fig:overview,fig:qualitative}, and a full description in~\Cref{sec:method}.
We conduct thorough experiments in~\Cref{sec:results} which show that \ours~sets a new SotA across the board.
Code, model weights, and data will be made publicly available.

\noindent To summarize, our \textbf{main contributions} are:
\begin{enumerate}[label=\textbf{\alph*)}]
\item We propose a reinforcement learning detection objective leveraging an aligned keypoint sampling procedure. This is described in~\Cref{sec:rl,sec:reward,sec:sampling,sec:regularization,sec:objective}.
\item We make empirical observations of emergent detector types from keypoint detection through reinforcement learning in~\Cref{sec:emerge}, which leads to
\item our proposed point-wise maximum distillation approach for training a more diverse detector in~\Cref{sec:distill}.
\item We obtain new state-of-the-art results for detectors over keypoint budgets spanning from 512 to 8192 keypoints, presented in~\Cref{sec:results}.
\end{enumerate}

\newpage

\section{Background}
\label{sec:background}
The goal of Structure-from-Motion (SfM) is to estimate the 3D structure of a scene in the form of a point cloud 
 and a set of cameras, 
from a (typically unordered) set of RGB images 
 \begin{equation}
     \mathcal{I} = \{\mathbf{I}^n \in \mathbb{R}^{H^n \times W^n \times 3}\}^N_{n=1}.
 \end{equation}
The most common approach to solve this problem is to detect, for each image $\mathbf{I}^n$, a set of $K^n$ keypoints 
\begin{equation}
\mathcal{K}^n = \{\mathbf{x}_{k}^n\in \mathbb{R}^2\}_{k=1}^{K^n} \sim p_{\theta}(\mathbf{x}|\mathbf{I}^n).    
\end{equation}
The keypoints are matched with a two-view matcher $m_{\theta}$ as 
\begin{equation}
\label{eq:matching}
 \mathcal{M}^{A\leftrightarrow B} = m_{\theta}(p_{\theta}, d_{\theta}| \mathbf{I}^{A}, \mathbf{I}^{B}) \in \mathbb{R}^{M\times 2\times 2}.   
\end{equation}

From the matches $\mathcal{M}^{A\leftrightarrow B}$, two-view geometric models, such as a Fundamental matrix $\mathbf{F}^{A\to B}$, Essential matrix $\mathbf{E}^{A\to B}$, or a Homography $\mathbf{H}^{A\to B}$ can be robustly estimated by RANSAC~\citep{fischler1981}. The two-view geometries serve as the initialization for either incremental~\citep{schoenberger2016sfm} or global~\cite{pan2024glomap} SfM optimizers, the workings of which are beyond the scope of this paper. 

Nevertheless, we will note the fact that repeatable keypoints are \emph{key} for these optimizers to converge well, as the point cloud is intimately connected to the keypoints in the form of \emph{feature tracks} $\tau$, and longer tracks, which are enabled by consistent detection, provide much stronger constraints on the 3D model. Note also that for practical reasons we can not select all pixels as keypoints, as this makes the optimization intractable.

We next describe our approach.

\section{Method}
\label{sec:method}
\subsection{Formulation of Keypoint Detection}
\label{sec:keypoint}
We aim to train a neural network $p_{\theta}$ that, conditioned on an image $\mathbf{I}^n$, infers a probability distribution, from which a set of keypoints $\mathcal{K}^n$ can be sampled.
In practice, the model produces a scoremap $\mathbf{S}  \in \mathbb{R}^{H\times W}$, which we view as logits of a probability distribution defined over the image grid and define the model keypoint distribution as
\begin{equation}
    p_{\theta}(\mathbf{x}|\mathbf{I}) \coloneq \text{softmax}(\mathbf{S}).
\end{equation}
At inference time we sample from this distribution, i.e.,
\begin{equation}
    \mathcal{K} \sim p_{\theta},
\end{equation}
by some (often deterministic) sampler.
The goal of keypoint detection from this perspective then, is to learn a distribution $p_{\theta}$, that when sampled, maximizes the quality of the resulting SfM reconstruction. 
However, this is not straightforward, as we discuss next.
\subsection{Keypoints via Reinforcement Learning}
\label{sec:rl}
Defining keypoints is difficult.
In principle they could be defined as points that, when jointly detected, give optimal reconstruction quality. Unfortunately, backpropagating through SfM pipelines is in general intractable.
However, we may be able to define a suitable reward, $r$, for detections that we by some metric deem good.
Typically one chooses to maximize the expected reward
\begin{equation}
    \max_{\theta}\mathbb{E}_{\tau\sim p_{\theta}}[r].
\end{equation}

We would thus like to differentiate this objective, i.e.,
\begin{equation}
    \nabla_{\theta}\mathbb{E}_{\tau\sim p_{\theta}}[r].
\end{equation}
We can rewrite the objective as
\begin{equation}
    \nabla_{\theta}\mathbb{E}_{\tau\sim p_\theta}[r] = \nabla_{\theta}\int_{\tau}rp_{\theta}(\tau)d\tau = \int_{\tau}r\nabla_{\theta}p_{\theta}(\tau)d\tau,
\end{equation}
where $\tau$ are feature tracks.
Using $\nabla \log p = \frac{\nabla p}{p}$ we have
\begin{equation}
    \nabla_{\theta}\mathbb{E}_{\tau\sim p_\theta}[r] = \int_{\tau}r(\tau)p_{\theta}(\tau)\nabla_{\theta}\log p_{\theta}(\tau)d\tau.
\end{equation}
This we can rewrite in expectation form as
\begin{align}
    \mathbb{E}_{\tau\sim p_\theta}[r(\tau)\nabla_{\theta}\log p_{\theta}(\tau)]. %
\end{align}
In this paper we will consider a modified version of this objective, following S-TREK~\citep{santellani2023strek}, where we take the expectation over another policy $q$ as
\begin{align}
    \mathbb{E}_{\tau\sim q}[r(\tau)\nabla_{\theta}\log p_{\theta}(\tau)]. %
\end{align}
The implications of on-policy vs off-policy is further discussed in~\Cref{suppl:off-policy}.
We will consider the two-view version, which reads
\begin{align}
    \sum_{m=1}^Mr(\mathbf{x}_m^A, \mathbf{x}_m^B) \nabla_{\theta}(\log p_{\theta}(\mathbf{x}_m^A|\mathbf{I}^A) + \log p_{\theta}(\mathbf{x}_m^B|\mathbf{I}^B)).
\end{align}
where $\{(\mathbf{x}_m^A, \mathbf{x}_m^B)\}_{m=1}^M = \mathcal{M}^{A\rightarrow B} \sim q$.
We can rewrite this as a loss as
\begin{equation}
    \mathcal{L}_{\rm RL} = \sum_{m=1}^Mr(\mathbf{x}_m^A, \mathbf{x}_m^B) (\log p_{\theta}(\mathbf{x}_m^A|\mathbf{I}^A) + \log p_{\theta}(\mathbf{x}_m^B|\mathbf{I}^B)).
\end{equation}
While this objective is very simple in principle, in practice both the design of the reward function and the sampling of $\mathcal{M}^{A\rightarrow B}$ makes a significant difference on the quality of the resulting models. 
We next go into detail on how we choose the reward, followed by the sampling and matching.

\subsection{Reward Function}
\label{sec:reward}
We use a keypoint repeatability-based reward which is computed by the pixel distance between the detections as
\begin{align}
    r(\tau) = r(\mathbf{x}_m^A, \mathbf{x}_m^B) = f(\lVert\mathcal{P}^{A\to B}(\mathbf{x}_m^A, z^A), \mathbf{x}_m^B \rVert).
\end{align}
where $\mathcal{P}^{A\to B}$ is a function transferring points between $\mathbf{I}^A$ and $\mathbf{I}^B$, $z^A$ are the image depths (only known/used during training), and $f$ is a monotonically decreasing function. In practice, we choose
\begin{equation}
    f(d) = \begin{cases}
        1, \quad d < \tau, \\
        0, \quad \text{else,}
    \end{cases}
\end{equation}
where we set $\tau$ to $0.25 \%$ of the image height. 
We additionally experimented with a linearly decreasing reward, similar to the one presented by~\citet{santellani2023strek}, which gives a large reward for very accurate keypoints, a smaller reward for less accurate keypoints, and reaches $0$ at $\tau$, but found that while such a reward does produce keypoints that are more sub-pixel accurate, it produces worse pose estimates, possibly due to being less diverse.
Finally, we normalize the reward over the image
\begin{equation}
\label{eq:fraction}
        r_{\text{pair}} = \frac{r}{\mathbb{E}[r]+\varepsilon},
\end{equation}
where the expectation is taken over each image pair $\mathbf{I}^A,\mathbf{I}^B$, and we set $\varepsilon=0.01$ to ensure stable gradients (if $\mathbb{E}[r] \approx 0$ the gradient will be arbitrarily large). 
This approach is related to the \emph{advantage}, typically defined as
\begin{equation}
    a = r - \mathbb{E}[r].
\end{equation}
However, we found the fraction form defined in \Cref{eq:fraction} to work better than the subtractive form.
We believe this is due to the negative rewards ``pushing'' down the sampling probabilities of rare keypoints, leading to early exploitation and failure to learn more generalizable points. 
We additionally experimented with incorporating a pose-based reward, which is described in more detail in \Cref{suppl:pose-reward}. 
However, we found the gains from incorporating it to be negligible, and thus use only the repeatability-based reward for our main approach.

\subsection{Keypoint Sampling and Matching}
We now describe how to sample $\mathcal{K}$, and consequently $\mathcal{M}^{A\rightarrow B}$, from $p_{\theta}$. A naive way is to draw $K$ samples independently at random, i.e., on-policy learning. This is problematic, as this can lead to distributional collapse to a single keypoint, or a cluster of keypoints. We would thus like the sampling to ensure three main properties
\begin{enumerate}
    \item \textbf{Sparsity:} The rewarded keypoints must be sparse. Detecting every point within an area, or on a line, does not constitute repeatable keypoints.
    \item \textbf{Diversity:} We want to detect a sufficiently large set of diverse keypoints, covering the scene.
    \item \textbf{Priority:} If a keypoint dominates another keypoint, i.e., its probability of detection is larger than the other keypoint in both images, we want this keypoint to always be prioritized over the other keypoint.
\end{enumerate}

\paragraph{Sampling Strategy.}
\label{sec:sampling}
To this end we find that the inference sampling strategy used in DeDoDe v2~\citep{edstedt2024dedodev2} is surprisingly effective.
Note that we use this sampling during \emph{both} training and inference, while it is used only during inference in DeDoDe v2.
We recap the strategy here.

We first estimate a smoothed kernel density estimate of the $p_{\theta}(\mathbf{x})$ as 
\begin{equation}
    p^{g}_{\theta}(\mathbf{x}) = (p_{\theta} * g)(\mathbf{x}),
\end{equation}
where $*$ is the convolution operator, and $g(\mathbf{x})$ is a Gaussian filter with standard deviation $\approx 2 \%$ of the image.

This is used to produce a balanced distribution \begin{equation}
    p^{\text{KDE}}_{\theta}(\mathbf{x}) \propto p_{\theta}(\mathbf{x}) \cdot p^{g}_{\theta}(\mathbf{x})^{-1/2}.
\end{equation}

This distribution is additionally filtered by non-local-maximima-supression (NMS) using a window size of 3 as
\begin{equation}
    q(\mathbf{x}) \propto \text{NMS}(p^{\text{KDE} }_{\theta}(\mathbf{x})).
\end{equation}
Finally, from this distribution we deterministically sample the top-K highest scoring pixels as $\mathcal{K} = \text{top-K}(q)$.

This process fulfills our desired properties of sampling. 
NMS makes sure that the detector does not degenerate to detect entire areas or lines (sparse).
The KDE downweighting makes sure that all parts of the image get keypoints sampled (diverse).
Perhaps more subtly, top-K sampling ensures the third property (priority). 
While we could in principle sample $q$ randomly, this causes unnecessary variance in the estimate. 
To see this, consider a toy-case where we have one ``true'' keypoint, and the rest is noise. 
As long as the probability of the true keypoint is $p<1$ we have a $(1-p)^K$ probability of not sampling the true keypoint. When $p\ll1$, as is common, this means that there is a large probability that the true keypoint would not be sampled. 
Using top-K sampling solves this problem.
As for the choice of $K$ itself, we simply set it as $K=512$, which we found to work well in general.

\paragraph{Matches $\mathcal{M}$.}
Once keypoints $\mathcal{K}^A$ and $\mathcal{K}^B$ have been sampled, we match these using the depth $z^A$ and $z^B$. 
Specifically, for each keypoint sampled in $\mathbf{I}^A$ we find its nearest neighbour in $\mathbf{I}^B$ and vice-versa. As in DISK~\citep{tyszkiewicz2020disk}, this results in $\mathcal{M}^{A\to B}$ and $\mathcal{M}^{B\to A}$, where we detach the gradient for $p_{\theta}(\mathbf{x}^B)$ for $\mathcal{M}^{A\to B}$ and vice-versa, and sum the two losses. In areas with no consistent depth (due to failure of MVS, or occlusion), we do not conduct any matching, and $\log p_{\theta}$ is computed with a masked log-softmax operation, to ensure that keypoints which are non-covisible are not punished.

\subsection{Regularization}
\label{sec:regularization}
In addition to the RL objective, we additionally employ regularization on the predicted scoremap as in DeDoDe~\citep{edstedt2024dedode}. Following DeDoDe we compute the regularization loss as
\begin{equation}
    \mathcal{L}_{\text{reg}}(p_{\theta}, p_{\rm depth}) = D_{\rm KL}(p_{\rm depth}*g \Vert p_{\theta} * g ),
\end{equation}
where $p_{\rm depth}$ is the per-image successful depth estimate indicator distribution, $g$ a Gaussian distribution with $\sigma \approx 12.5$ pixels, and $D_{\rm KL}$ the Kullback-Leibler divergence.

\subsection{Initial Learning Objective}
\label{sec:objective}
Our full learning objective is an unweighted combinatation of the RL objective and the regularization objective, as
\begin{equation}
\label{eq:main-objective}
    \mathcal{L}_{\text{full}} = \mathcal{L}_{\rm RL} + \mathcal{L}_{\rm reg}.
\end{equation}
We found that the performance as a function of the weight of the regularization objective was quite stable to up to about an order of magnitude.
However, it turns out that this objective is optimized by two qualitatively different types of detectors, which we discuss next.

\subsection{Emergent Detector Types}
\label{sec:emerge}
We find that several qualitatively different types of detectors emerge, when trained with \Cref{eq:main-objective} combined with rotation augmentation. 
\paragraph{Light and dark detectors:}
Our main focus will be on two types of detectors which we call \emph{light} and \emph{dark} detectors. A example of the types of keypoints these detectors find is shown in~\Cref{fig:dark-light}.
\begin{figure}
    \centering
    \includegraphics[width=\linewidth]{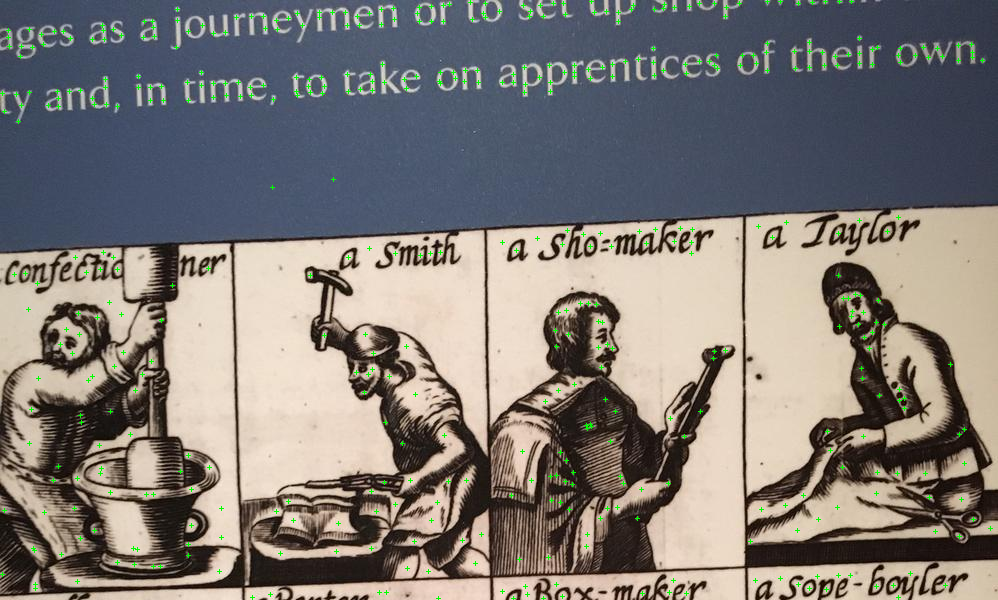}
    \includegraphics[width=\linewidth]{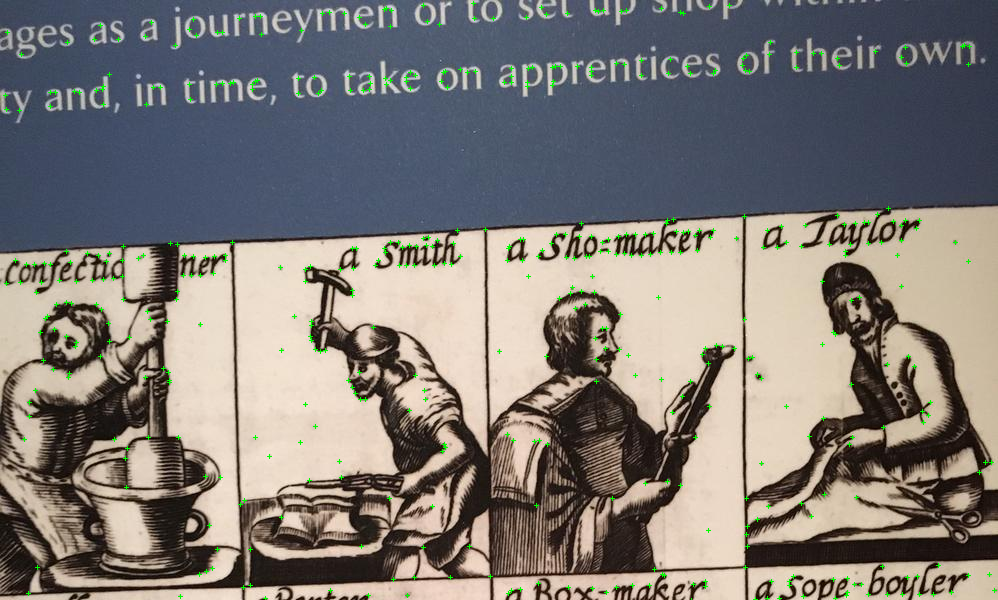}
    
    \caption{\textbf{Light vs dark keypoints.} Detectors can significantly increase the expected reward by choosing either \emph{dark} keypoints, where the pixel intensity is low, or \emph{light} keypoints, where the pixel intensity is high. It turns out that either of these choices produce approximately the same expected reward. However, we argue that this is an undesirable property, e.g., due to inversions that can occur naturally due to day-night changes, or that certain images may be dominated by either dark or light keypoints.}
    \label{fig:dark-light}
\end{figure}
We conjecture that this is an intrinsic property of the keypoint detection problem, whereby significant gains in repeatability can always be obtained by sacrificing diversity. 
We give a more extensive motivation for how this behavior may arise naturally in through analysis of a toy model of keypoint detection in~\Cref{suppl:toy-model}.

Perhaps even more surprisingly, we find that this not only occurs in our objective and model, but also in ALIKED~\citep{Zhao2023ALIKED}, which uses a different objective and architecture. 
In particular, empirically we find that it occurs when using rotation augmentation. 
We qualitatively demonstrate this finding for ALIKED in~\Cref{fig:upright-rot}.
\begin{figure}
    \centering
    \includegraphics[trim=0cm 26.5cm 50cm 0cm, clip,width=.95\linewidth]{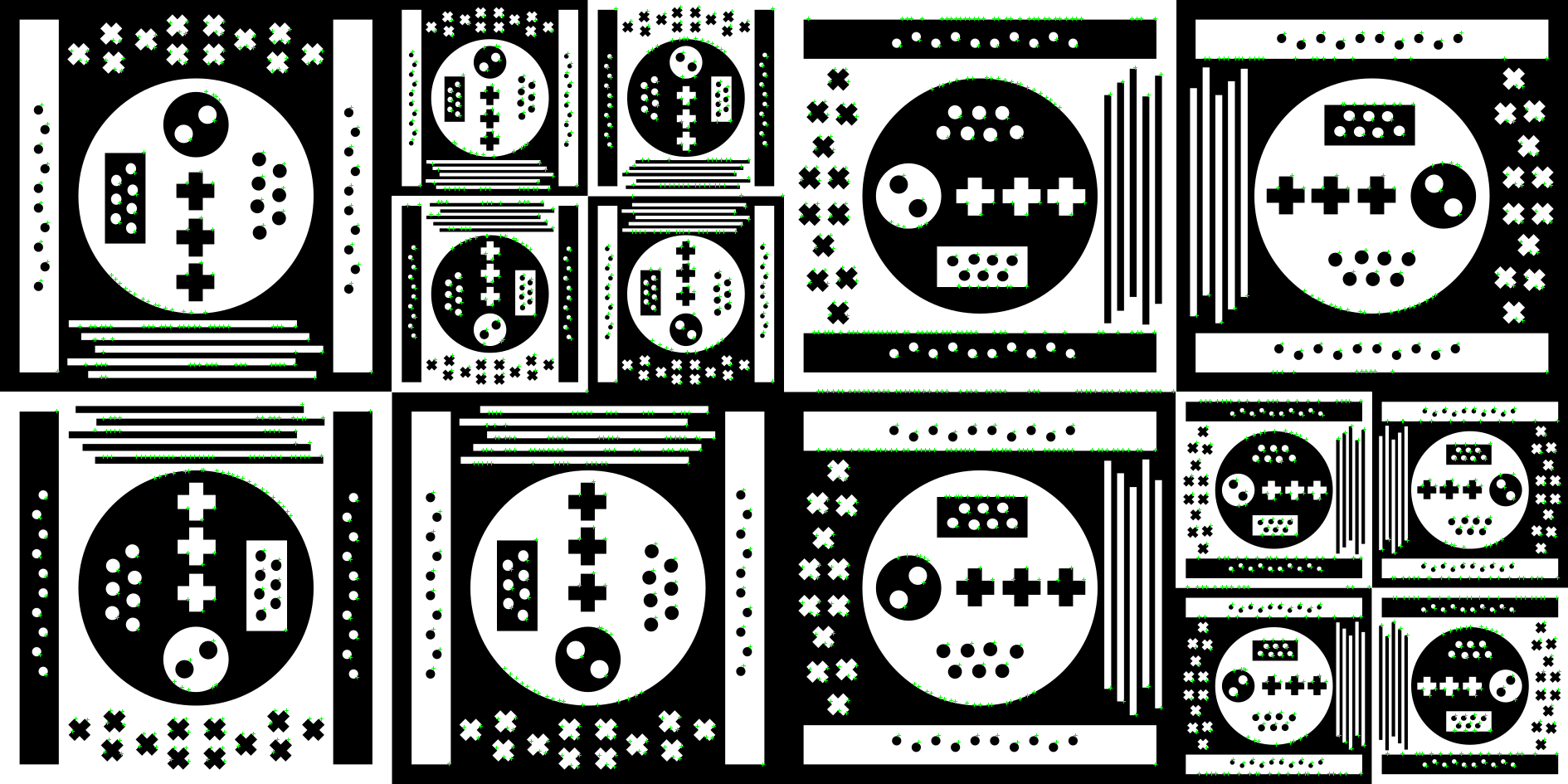}
    \includegraphics[trim=0cm 26.5cm 50cm 0cm, clip,width=.95\linewidth]{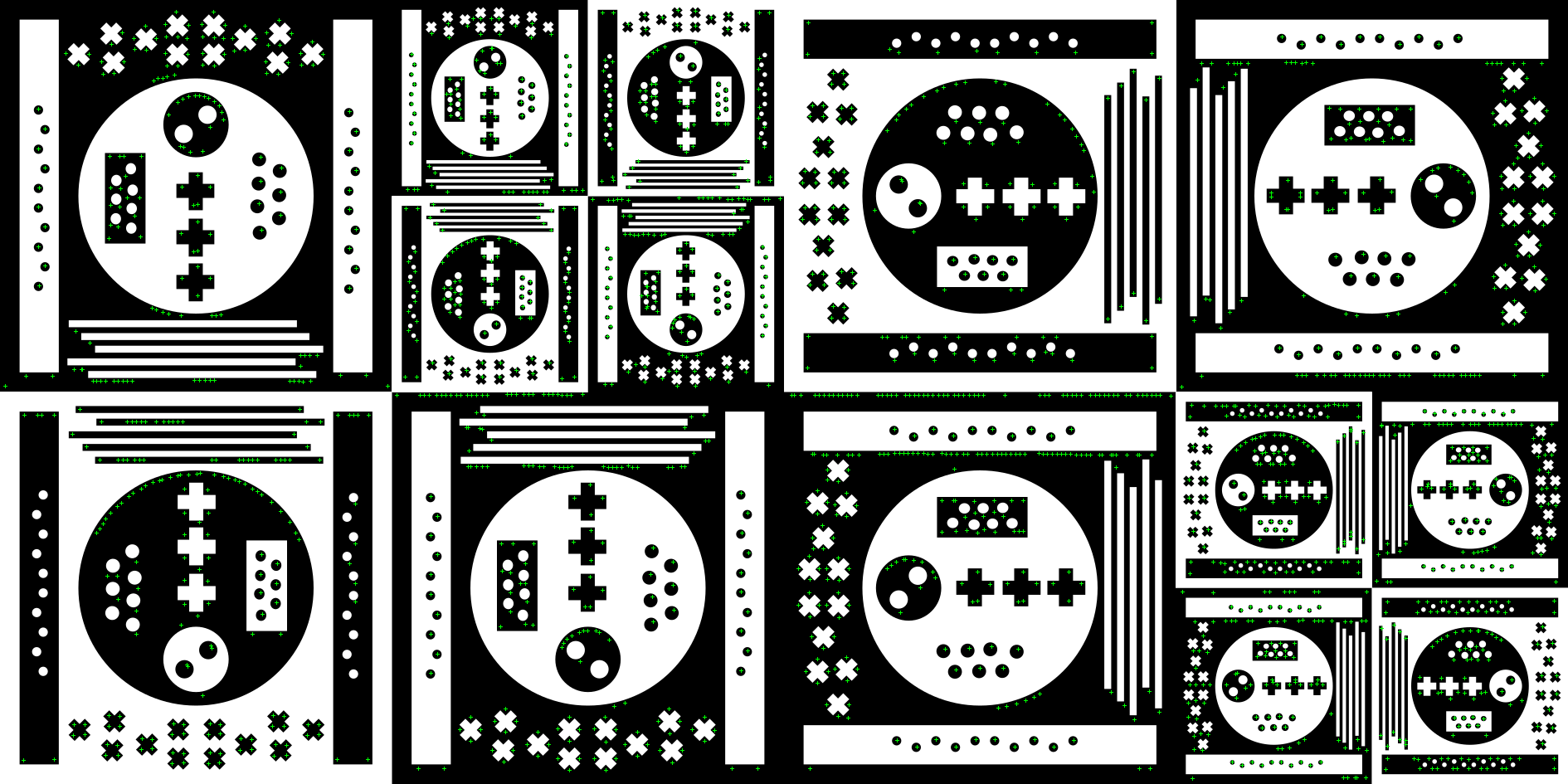}
    \caption{\textbf{Enforcing rotation invariance causes light/dark keypoint detectors also in ALIKED.}
    \textbf{Top:} Detections of ALIKED trained on upright images. \textbf{Bottom:} Detections of ALIKED trained with rotation augmentation. Remarkably, we observe that enforcing rotation invariance is what causes the emergence of light/dark keypoint detectors, and that this holds also for ALIKED, which uses a different objective and architecture than ours.}
    \label{fig:upright-rot}
\end{figure}
As can be seen from the figure, while the upright detector avoids the dark/light dilemma, it exhibits a strong upright bias. For example, it only detects keypoints on the lower right of white circles, which would only yield repeatable detections for upright images.
We were able to reproduce a similar behaviour in \ours~by removing the rotation augmentation.
In~\Cref{suppl:rekd} we also find REKD~\citep{lee2022self}, which uses a rotation equivariant architecture, only detects dark keypoints, showing that this behaviour extends beyond augmentation. 
We conducted experiments whether light/dark detection was correctable through augmentation (image negations), but found that it led to significant degradation in performance, which we describe in~\Cref{suppl:failed-aug}. 
In~\Cref{sec:distill} we instead propose a way to harness both detector types through distillation.

\paragraph{Bouba/Kiki detectors:}
We also observe, with naming inspired by the bouba/kiki phenomenon~\citep{ramachandran2001synaesthesia}, that detectors may converge to \emph{bouba}-type detectors, typically placing the keypoint maximum at the center of a blob, and \emph{kiki}-type detectors, which place keypoints only at corners. 
A qualitative example of these two types is presented in~\Cref{suppl:bouba-kiki}.
Empirically we found that the kiki-type detector performs better, and therefore, by visual inspection of detectors from different random seeds, select a dark detector exhibiting kiki-type behaviour. 

\subsection{Light+Dark Detector Distillation}
\label{sec:distill}
In order to train a detector capabale of detecting both light and dark types of keypoints, we train a third detector with the following distillation objective
\begin{equation}
    \mathcal{L}_{\text{distill}} = D_{\rm KL}(p_{r} \Vert p_{\theta}),
\end{equation}
where 
\begin{equation}
    p_{r}(\mathbf{x}) \propto M_r(p_{\theta_{\text{dark}}}(\mathbf{x}),p_{\theta_{\text{light}}}(\mathbf{x})),
\end{equation} 
and
\begin{equation}
    M_r(a,b) = (1/2 (a^r+b^r))^{1/r}
\end{equation}
is a Generalized mean function, where we consider $r\in\{1,2,\infty\}$.
Empirically we found that $r=\infty$, which corresponds to taking the point-wise maximum of the two distributions, performs the best.
An illustration for why using $r=\infty$ is reasonable for keypoints is given in~\Cref{fig:max}.
We additionally prove that the local maxima are preserved if $r=\infty$ under some mild assumptions in~\Cref{suppl:max}.
\begin{figure}
    \centering
    \includegraphics[width=\linewidth]{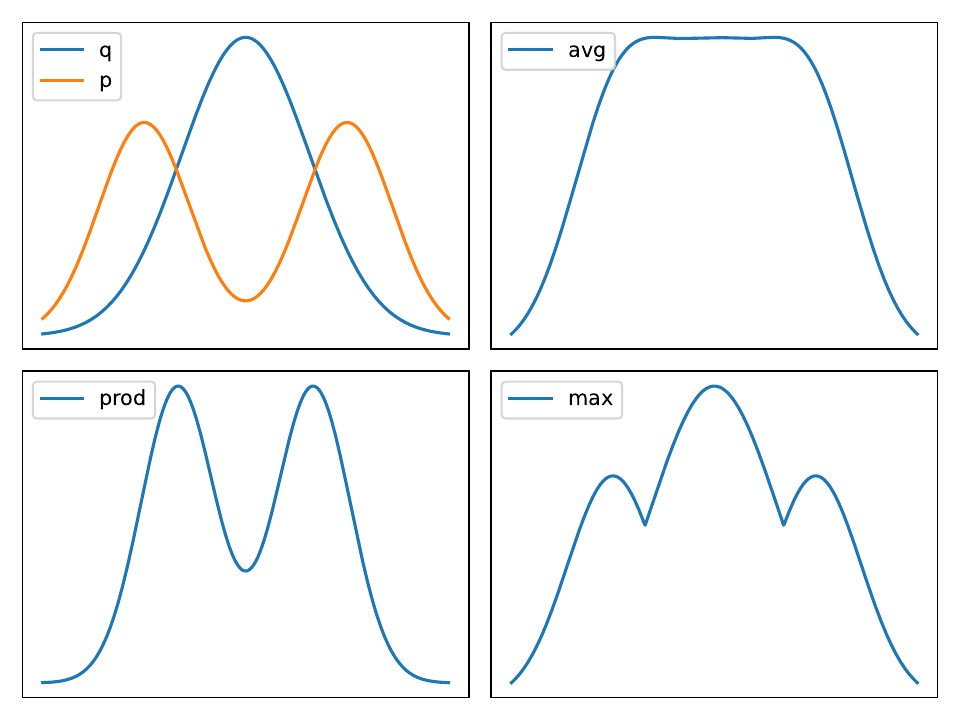}
    \caption{\textbf{Why max is good.} Given two detectors $p$ and $q$, we would like their ensemble to retain the original keypoints. While averaging or multiplying the distributions typically change the shape and locations of the keypoints, the max operation preserves the peaks significantly better.}
    \label{fig:max}
\end{figure}

One might wonder if further gains could be made by using the RL objective with pretrained weights from the distilled model.
The answer is, surprisingly, no.
We empirically found that even models distilled to detect both light and dark keypoints degenerate to only detecting one or the other when finetuned with the RL objective.
This suggests that the light/dark emergence is a fundamental property of the keypoint detection problem.

\section{Experiments}
\label{sec:results}
\subsection{Network Architecture}
Inspired by the architecture of DeDoDe~\citep{edstedt2024dedode}, we use the \emph{DeDoDe-S} model, which uses a VGG11~\citep{simonyan2015vgg} backbone as feature encoder, and depthwise seperable convolution blocks as decoder. 
We provide full details for the architecture used in \Cref{suppl:arch}.
We found that while a larger architecture gives slightly higher performance, the extra gains were not sufficient to motivate the significantly increased runtime. Runtime comparisons between \ours~and other SotA models are presented in~\Cref{suppl:runtime}.
\subsection{Training Details}
We follow DeDoDe v2~\citep{edstedt2024dedodev2} and train on MegaDepth, using random $\{\ang{0},\ang{90}, \ang{180}, \ang{270}\}$ rotations, which is done to ensure equivariance in the case of non-upright images~\citep{bökman2024steerers, bokman2024affine}. We use a training resolution of 640. We use the AdamW optimizer with a learning rate of $2 \cdot 10^{-4}$ for the decoder and $1\cdot10^{-5}  $ for the encoder. 
We train the dark detector for 600k pairs, the light detector for 800k pairs, and the distillation of the two for another 800k pairs.
The training time for each detector was approximately 10 hours on an A100 GPU.

\subsection{Inference Settings}
\label{sec:inference-settings}
We follow the settings of LightGlue and resize the longer size of the image to 1024. We use an NMS size of $3$. We use a similar sampling strategy as during training, however we do not use the KDE balancing of the keypoints.%
We additionally use subpixel sampling inspired by ALIKED~\citep{Zhao2023ALIKED} where around each keypoint we compute a local softmax distribution (with the same size as the NMS) with temperature $\tau = 0.5$ and compute the adjusted keypoint as the expectation over the patch distribution.
Inference settings for other detectors and detailed evaluation protocol are presented  in~\Cref{suppl:inference,suppl:eval-protocol} respectively.
\begin{table}
    \centering
    \caption{\textbf{Ablation study.} Performance is measured as the average AUC@5$^{\circ}$ of the pose error over Essential and Fundamental matrix estimation on MegaDepthIMCPT (higher is better).}
    \begin{tabular}{l ccccc}
    \toprule
        \multirow{2}{*}{\underline{Method} $\downarrow$} & \multicolumn{5}{c}{\underline{Max \textnumero~Keypoints}} \\
        & 512 & 1024 & 2048 & 4096 & 8192 \\
        \midrule
         Light & 50.2 & 53.1 & 55.1 & 55.8 & 55.9 \\
         Dark & 50.0 & 52.9 & 55.0 & 55.3 & 55.7 \\
         \midrule
         $r=1$ & 50.7 & 53.7 & 55.8 & 56.3& 56.0 \\
         $r=2$ & 51.0 & 54.1 & 55.6 & 56.3 & 56.3 \\
         $r=\infty$ (\textbf{\ours}) & 51.0 & 53.8 & 55.8 & 56.5& 56.5\\
         \bottomrule
    \end{tabular}
    \label{tab:megaIMCPT}
\end{table}

\paragraph{Match Estimation:}
Evaluating the performance of a detector requires a matcher $m_{\theta}$. 
While for certain benchmarks dense ground truth correspondences exist, in general these are either too sparse (as in MegaDepth), or inaccurate (as in ScanNet). 
We thus want to use a matcher independent of the dataset.
To this end we leverage the dense image matcher RoMa~\citep{edstedt2024roma} to match the keypoints.
We use the following heuristic. 
For each keypoint in $\mathbf{I}^A$ we sample the dense RoMa warp, to get a corresponding point in $\mathbf{I}^B$. 
We then compute the distance between the warped points and the detections in $\mathbf{I}^B$. 
We define matches as mutual nearest neighbours within a distance of $0.25\%$ of the image size of each other. 
For RoMa we use a coarse resolution of $560\times 560$ and a $864\times 864$ upsample resolution, which are the defaults. 
We do not use symmetric matching, i.e., incorporating the backwards warp of RoMa, for simplicity.

\subsection{Ablation Study}
We investigate the performance of variations of our model on a validation set of eight Megadepth scenes corresponding to \emph{Piazza San Marco} (0008), Sagrada Familia (0019), Lincoln Memorial Statue (0021), British Museum (0024), Tower of London (0025), Florence Cathedral (0032), Milan Cathedral (0063), Mount Rushmore (1589).
We report the pose AUC as error metric. Results averaged over Essential and Fundamental matrix estimation are presented in~\Cref{tab:megaIMCPT}. We find that both the light and dark detector performs similarly, while our distilled model significantly outperform both, validating our approach.
We additionally find that using $r=\infty$, i.e., using pointwise maximum, generally outperforms $r=1$, i.e., averaging, by about $0.3$ points.

\subsection{State-of-the-Art Comparisons}
We compare \ours~to previous SotA detectors on standard two-view geometry estimation tasks (Essential matrix, Fundamental matrix, Homography). 
We compare against the supervised detectors SuperPoint~\citep{detone2018superpoint} and ReinforcedFP~\citep{bhowmik2020reinforced}.
We additionally compare to four learned self-supervised detectors, ALIKED~\citep{Zhao2023ALIKED} ($\circlearrowleft$ indicates rotational augmentation), REKD~\citep{lee2022self}, DISK~\citep{tyszkiewicz2020disk}, and DeDoDe v2~\citep{edstedt2024dedodev2}. 
For reference we additionally include SIFT~\citep{lowe2004distinctive}.

\paragraph{MegaDepth1500:}
MegaDepth1500 is a standard benchmark for image matching introduced in LoFTR~\citep{sun2021loftr}. It consists of 1500 image pairs with a uniform distribution of overlaps between 0.1 and 0.7 from the exterior of Saint Peter's Basilica (scene number 0015 in MegaDepth), and the front of Brandenburger Tor (scene number 0022 in MegaDepth). 
We evaluate both Essential matrix and Fundamental matrix estimation. 
We report the pose AUC@\ang{5} as error metric. Results for Essential matrix estimation are presented in \Cref{tab:mega1500} and Fundamental matrix estimation in~\Cref{tab:mega1500_F}. \ours~clearly outperforms all previous methods, including the supervised SuperPoint, with a wide margin for all number of keypoints.
\begin{table}[]
    \centering
    \caption{\textbf{Essential Matrix Estimation on MegaDepth1500.} Performance is measured as the AUC@5$^{\circ}$ of the pose error (higher is better). Upper portion contains supervised detectors.}
    \begin{tabular}{l ccccc}
    \toprule
        \multirow{2}{*}{\underline{Method} $\downarrow$} & \multicolumn{5}{c}{\underline{Max \textnumero~Keypoints}} \\
        & 512 & 1024 & 2048 & 4096 & 8192 \\
        \midrule
         SuperPoint & 57.5 & 63.3 & 65.6 & 65.7 & 65.0 \\
         ReinforcedFP & 58.7&63.0&64.7&65.2&64.2\\
         \midrule
         SIFT & 48.9 & 58.8 & 64.3 & 65.9 & 65.6 \\
         ALIKED & 63.8 & 66.0 & 66.9 & 67.5 & 67.8 \\
        ALIKED$^\circlearrowleft$ & 63.1 & 65.7& 67.3& 68.3& \first{68.8}\\
        REKD & 59.8	&65.7	&66.7	&66.7	&66.6 \\
         DISK & 55.4 & 59.5 & 62.2 & 65.0 & 65.9 \\
         DeDoDe v2 & 57.5 & 63.8 & 66.6 & 68.1 & 68.0 \\
         \textbf{\ours} (Ours) & \first{64.9} & \first{67.9} & \first{69.4} & \first{69.3} & \first{68.8} \\
         \bottomrule
    \end{tabular}
    \label{tab:mega1500}
\end{table}

\begin{table}[]
    \centering
    \caption{\textbf{Fundamental Matrix Estimation on MegaDepth1500.} Performance is measured as the AUC@5$^{\circ}$ of the pose error (higher is better). Upper portion contains supervised detectors.}
    \begin{tabular}{l ccccc}
    \toprule
        \multirow{2}{*}{\underline{Method} $\downarrow$} & \multicolumn{5}{c}{\underline{Max \textnumero~Keypoints}} \\
        & 512 & 1024 & 2048 & 4096 & 8192 \\
        \midrule
        SuperPoint & 41.7 & 48.3& 51.4 & 50.6& 49.5 \\
        ReinforcedFP & 42.7 & 46.9& 48.5& 49.2& 47.8 \\
        \midrule
        SIFT & 36.7& 45.7& 50.7 & 52.5& 52.4\\
        ALIKED & 49.7 & 53.2& 54.7& 55.6& 56.2 \\
        ALIKED$^\circlearrowleft$ & 48.6 & 52.3& 54.1 & 56.4& \first{57.6} \\
        REKD & 44.1	&50.8	&53.3	&53.3	&53.3\\
        DISK & 38.8 & 45.0& 49.0& 52.1& 53.6 \\
        DeDoDe v2 & 40.8 & 48.3& 53.1& 54.5& 54.6 \\
        \textbf{\ours} (Ours) & \first{50.6} & \first{55.5} & \first{57.4} & \first{57.8} & 56.4\\
         \bottomrule
    \end{tabular}
    \label{tab:mega1500_F}
\end{table}

\paragraph{ScanNet1500:}
ScanNet1500 was introduced in SuperGlue~\citep{sarlin20superglue} and consists of 1500 pairs between $[0.4, 0.8]$ overlap (computed via GT depth) from the ScanNet~\citep{dai2017scannet} dataset. As in MegaDepth1500 we report the pose AUC as our main metric. Results are presented in \Cref{tab:scannet1500,tab:scannet1500_F}. \ours~sets a new state-of-the-art. In particular, we find that our model significantly outperforms previous self-supervised detectors for Fundamental matrix estimation.
\begin{table}[]
    \centering
    \caption{\textbf{Essential Matrix Estimation on ScanNet1500.} Performance is measured as the AUC@5$^{\circ}$ of the pose error (higher is better). Upper portion contains supervised detectors.}
    \begin{tabular}{l ccccc}
    \toprule
        \multirow{2}{*}{\underline{Method} $\downarrow$} & \multicolumn{5}{c}{\underline{Max \textnumero~Keypoints}} \\
        & 512 & 1024 & 2048 & 4096 & 8192 \\
        \midrule
        SuperPoint & 24.7 & 26.4 & 28.0 & 28.8 & 28.3\\
         ReinforcedFP & 25.2 & 27.3 & 28.8 & 28.4 & 28.9 \\
         \midrule
         SIFT & 17.9 & 23.5 & 26.2 & 27.5 & 27.8 \\
         ALIKED & 22.3 & 24.1 & 25.9 & 28.0 & 28.3 \\
        ALIKED$^\circlearrowleft$ & 21.8 & 24.5 & 26.9 & 27.4 & 27.8 \\
        REKD & 22.0&	25.1&	26.2&	26.3&	26.3\\
         DISK & 15.2 & 20.0 & 23.5 & 25.9 & 27.2 \\
         DeDoDe v2 & 19.9 & 25.2 & 27.4 & 28.6 & 29.2 \\
         \textbf{\ours} (Ours) & \first{25.7} & \first{27.9} & \first{28.3} & \first{28.9} & \first{29.3} \\
         \bottomrule
    \end{tabular}
    \label{tab:scannet1500}
\end{table}

\paragraph{HPatches:}
HPatches~\citep{balntas2017hpatches} is a Homography benchmark of a total of 116 sequences, of which 57 consist of variation in illumination, and 59 sequences consist of viewpoint changes. The task is to estimate the dominant planar Homography from the keypoint correspondences (identity in the case of illumination). We consider mainly viewpoint changes, results of which are presented in~\Cref{tab:hpatches}. We find that \ours~sets a new state-of-the-art, particularly in the few-keypoint setting.

\begin{table}[]
    \centering
    \caption{\textbf{Fundamental Matrix Estimation on ScanNet1500.} Performance is measured as the AUC@5$^{\circ}$ of the pose error (higher is better). Upper portion contains supervised detectors.}
    \begin{tabular}{l ccccc}
    \toprule
        \multirow{2}{*}{\underline{Method} $\downarrow$} & \multicolumn{5}{c}{\underline{Max \textnumero~Keypoints}} \\
        & 512 & 1024 & 2048 & 4096 & 8192 \\
        \midrule
        SuperPoint & 15.7 & 18.1 & 20.2 & 22.6 & 23.2 \\
        ReinforcedFP & 17.0 & 19.1 & 20.2 & 22.6 & 23.2 \\
        \midrule
        SIFT & 10.8 & 14.9 & 18.7 & 19.6 & 19.6 \\
        ALIKED & 14.2 & 16.5 & 18.4 & 21.6& 22.9 \\
        ALIKED$^\circlearrowleft$ & 14.2 & 17.5 & 19.7 & 21.1& 21.6 \\
        REKD & 13.3&	16.4&	18.5&	18.5&	18.5\\
        DISK & 6.9 & 11.0 & 15.6 & 19.1 & 21.0\\
        DeDoDe v2 & 10.3 & 16.0& 19.7 & 21.9 & 22.9\\
         \textbf{\ours} (Ours) & \first{18.3} & \first{20.6} & \first{21.7} & \first{22.2} & \first{23.0} \\
         \bottomrule
    \end{tabular}
    \label{tab:scannet1500_F}
\end{table}

\begin{table}[]
    \centering
    \caption{\textbf{Homography Estimation on HPatches.} Performance is measured as the AUC@3 pixels of the end-point error (higher is better). Upper portion contains supervised detectors, and lower self-supervised detectors.}
    \begin{tabular}{l ccccc}
    \toprule
        \multirow{2}{*}{\underline{Method} $\downarrow$} & \multicolumn{5}{c}{\underline{Max \textnumero~Keypoints}} \\
        & 512 & 1024 & 2048 & 4096 & 8192 \\
        \midrule
         SuperPoint & 55.0 & 57.5 & 59.3 & 60.5 & 61.6 \\
        ReinforcedFP & 55.8 & 58.5 & 60.9 & 62.8 & 62.9 \\
         \midrule
         SIFT & 51.7 & 53.4 & 55.9 & 58.1 & 58.8 \\
         ALIKED & 52.6 & 54.8 & 56.4 & 57.9 & 59.1 \\
         ALIKED$^\circlearrowleft$ & 48.2 & 52.4 & 54.0 & 56.3 & 57.9 \\
         REKD & 49.6 & 54.8 & 56.3& 56.2& 56.2\\
         DISK & 44.8 & 49.9 & 54.0 & 55.9 & 58.0 \\
         DeDoDe v2 & 52.2 & 56.9 & 60.0 & \first{61.8} & \first{62.9} \\
         \textbf{\ours} (Ours) & \first{58.2} & \first{60.5} & \first{61.0} & 61.5 & 61.4 \\
         \bottomrule
    \end{tabular}
    \label{tab:hpatches}
\end{table}

\section{Conclusion}
\label{sec:conclusion}
 
We presented \ours, a descriptor-free keypoint detector trained using reinforcement learning. We tackled how to formulate the keypoint detection problem in a policy gradient framework, how to sample diverse keypoints for training, regularization, and investigated two qualitatively different types of detectors that emerge from RL and how these can be optimally merged.
Finally, we conducted an in-depth set of experiments comparing a range of recent and traditional keypoint detectors on descriptor-free two-view pose estimation, which show that \ours~perform competitively or sets a new SotA on all benchmarks. 
In particular, \ours~excels in the few-keypoint setting, where previous self-supervised detectors often struggle. 

\paragraph{Limitations \& Future Work:}
\begin{enumerate}[label=\textbf{\alph*)}]
    \item While we identify several types of emergent detectors, which we combine by distillation, our proposed distillation objective does not directly optimize keypoint repeatability. Finding a way to optimize the repeatability objective while ensuring diversity is an interesting future direction.
    \item Further, while out-of-scope for the present paper, designing neural network architectures that are invariant under the observed keypoint variations (\eg, light-dark) is a potential avenue to enable the detector to find all types of keypoints, without distillation.
\end{enumerate}
\newpage

\section*{Acknowledgements}
This work was supported by the Wallenberg Artificial
Intelligence, Autonomous Systems and Software Program
(WASP), funded by the Knut and Alice Wallenberg Foundation and by the strategic research environment ELLIIT, funded by the Swedish government. The computational resources were provided by the
National Academic Infrastructure for Supercomputing in
Sweden (NAISS) at C3SE, partially funded by the Swedish Research
Council through grant agreement no.~2022-06725, and by
the Berzelius resource, provided by the Knut and Alice Wallenberg Foundation at the National Supercomputer Centre.
{
    \small
    \bibliographystyle{ieeenat_fullname}
    \bibliography{main}
}

\clearpage
\setcounter{page}{1}
\appendix

\maketitlesupplementary

Here we provide additional details and derivations that did not fit the main text.
\section{Qualitative Examples of Previous Self-Supervised Rotation Invariant Detectors}
\label{suppl:qualitative-comparison}
In contrast to the qualitative example presented in the main text of \ours~detections (see~\Cref{fig:qualitative}), we here present qualitative examples of two previous self-supervised rotation invariant detectors in~\Cref{fig:qual-aliked,fig:qual-rekd}. As can be observed in the figures, there is not a single keypoint placed on light pixels. 

\section{Off-Policy Reinforcement Learning for Keypoints}
\label{suppl:off-policy}
Policy gradient is per-definition on-policy as it is the expectation is in over trajectories samples drawn from the policy.
However, keypoint detection introduces the constraint of always sampling a fixed number of unique trajectories, which poses challenges to the on-policy paradigm.

DISK~\citep{tyszkiewicz2020disk}, which is still on-policy during training, enforces the constraint implicitly by first sampling one proposal keypoint per patch, and then accepting/rejecting this keypoint based on an additional distribution.
As only a keypoint can only be drawn once, this ensures the learning does not collapse.
While this makes the learning on-policy, it has issues with, e.g., artifacts around patch boundaries~\citep{santellani2023strek}.

In this paper we avoid using patch-based sampling, and instead emulate the procedure used at inference time as it is better aligned with our goals, but makes our approach off-policy. Our sampling has similarities to the sampling used in S-TREK~\citep{santellani2023strek}, that sequentially samples in such a way that no sample is within a radius $r$ of each other, which is reminiscent of Poission-disk sampling~\citep{cook1986stochastic, bridson2007fast}, however, their procedure will still face issues when many points have similar probability, which is not the case with top-k sampling.

\section{Pose Reward}
\label{suppl:pose-reward}
The end goal of keypoint detection is improving the final reconstruction quality, which is both non-differentiable, and would be intractably time-consuming.
However, we may set up a reward for smaller SfM-like problems, and set a reward for reconstructions that produce a smaller reconstruction error. 

To emulate the test-time setup, we draw several subsets of $\mathcal{K}_i$, $K_i=10$ keypoints, and $i \in \mathbb{S} = \{1,\dots,|\mathcal{K}|/|\mathcal{K}_i|\}$.
\begin{equation}
    \epsilon_{\mathbf{R}}(\hat{\mathcal{K}}) = \epsilon_{\mathbf{R}}(\text{RANSAC}(\hat{\mathcal{K}})),
\end{equation}
where we run PoseLib for $10$ iterations, and include refinement. We use only the rotational error, as the translational error is poorly conditioned for small baselines due to the scale ambiguity.

We choose not to base the reward directly on the pose-error, as the poses and cameras of MegaDepth have modeling errors~\citep{brissman2023camera}, and are thus prone to unreliable estimates of reconstruction quality. However, we might expect that better solutions have lower pose error in general. 
\begin{equation}
    r_{\text{pose}}(\mathcal{K}_i) = \begin{cases}
        1 - \frac{\epsilon_{\mathbf{R}}(\mathcal{K}_i)}{\epsilon_{\text{max}}} , \epsilon_{\mathbf{R}}(\mathcal{K}_i) \le \epsilon_{\text{max}} \forall j \in \mathbb{S},\\
        0, \text{else,}
    \end{cases}.
\end{equation}
where we set $\epsilon_{\text{max}}=10$ in practice.

However, as discussed in~\Cref{sec:reward}, we experimentally found that this reward did not improve the results significantly, and thus do not retain it in our final approach. 
While similar approaches has previously been shown to work by~\citet{bhowmik2020reinforced}, they initialize their model from a pretrained SuperPoint model. 
It is possible that such finetuning could also be applied to our detector as well, but it is beyond the scope of the current paper.

\section{A Toy Model for How Dark/Bright Detectors Emerge}
\label{suppl:toy-model}
Consider $H\times W$ images with two types of keypoints, white pixels and black pixels.
In each image there are 10 white pixels, and 10 black pixels.
For completeness, we can consider the rest of the pixels as gray, and that $H\cdot W \gg 10$. 
An example image is shown in~\Cref{fig:toy-model}.

The locations of the keypoints are completely random, and are indistinguishable from each other. 
However, they have a corresponding keypoint in another image (with a different random spatial distribution of the keypoints). If a pair of corresponding keypoints are selected, the model gets a reward of 1. The total reward is the total number of such selected keypoint pairs.
\begin{figure}
    \centering
    \includegraphics[width=\linewidth]{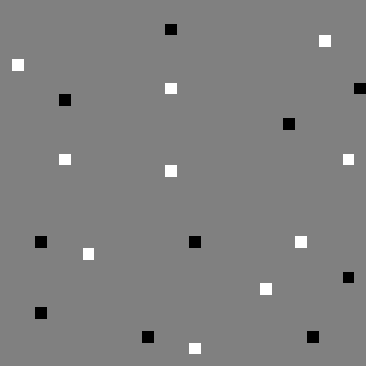}
    \caption{\textbf{Toy model of keypoint detection.} Under a constraint of $K=10$ keypoints and a random spatial distribution of the keypoints, the only consistent way to choose the keypoints is to either only choose white, or only choose black.}
    \label{fig:toy-model}
\end{figure}

We now set a constraint, that a maximum of 10 keypoints can be selected from each image. 
Three plausible strategies are,
\begin{enumerate}
    \item Select 5 white and 5 black keypoints randomly.
    \item Select only white keypoints.
    \item Select only black keypoints.
\end{enumerate}
If we compute the expected reward for these strategies, we get $\{5,10,10\}$. Clearly then, the optimal solution will involve choosing either only bright or only dark keypoints.

We conjecture that this reasoning extends beyond the toy case, leading to detectors very quickly choosing either bright or dark keypoints. While this might seem unproblematic (after all, if it is more consistent, why not choose), we find that in practice certain images may consist of only one type of keypoints, or that the spatial distribution is better when considering both types.

While one could in principle impose additional rewards for the distribution of keypoints, e.g. by KDE weighting, and regularization as we do, we found that the amount of regularization needed made the detectors significantly worse.

\section{REKD is a Dark Detector}
\label{suppl:rekd}
In~\Cref{fig:rekd-dets} we qualitatively demonstrate that REKD is a dark detector.
This is interesting as REKD does not use rotational augmentation, but rather use a equivariant architecture. 
This indicates that dark/light emergence is independent on if rotation invariance is enforced through augmentation or through architecture.

\begin{figure*}
    \centering
    \includegraphics[width=\linewidth]{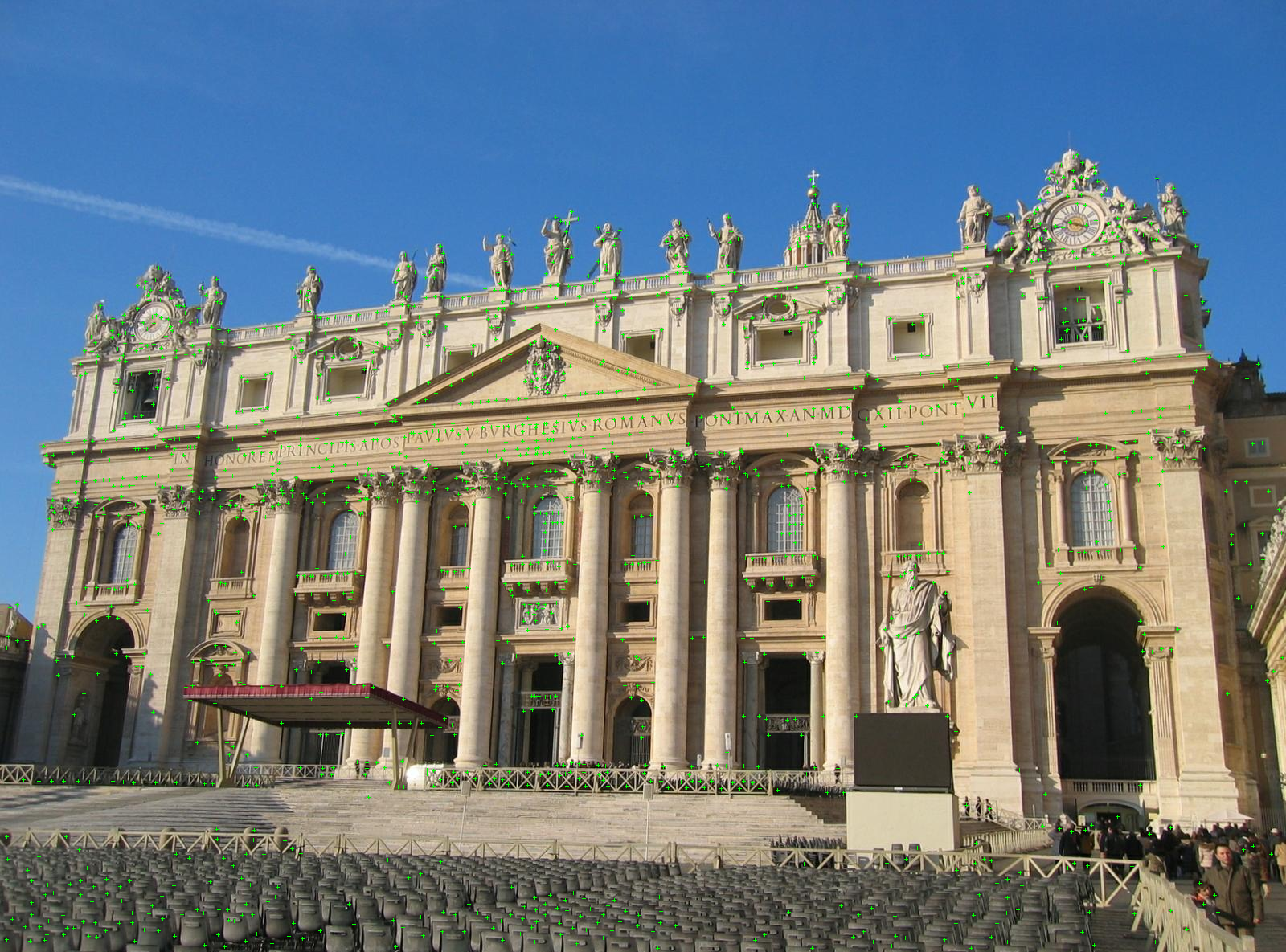}
    \caption{\textbf{Qualitative example of ALIKED (trained with rotational augmentations) keypoint detections.} In contrast to \ours, while there are keypoints near the cross, there are no keypoints on the cross. If one observes the figure closely, it can be observed that there in fact no keypoint on light colored pixels.}
    \label{fig:qual-aliked}
\end{figure*}

\begin{figure*}
    \centering
    \includegraphics[width=\linewidth]{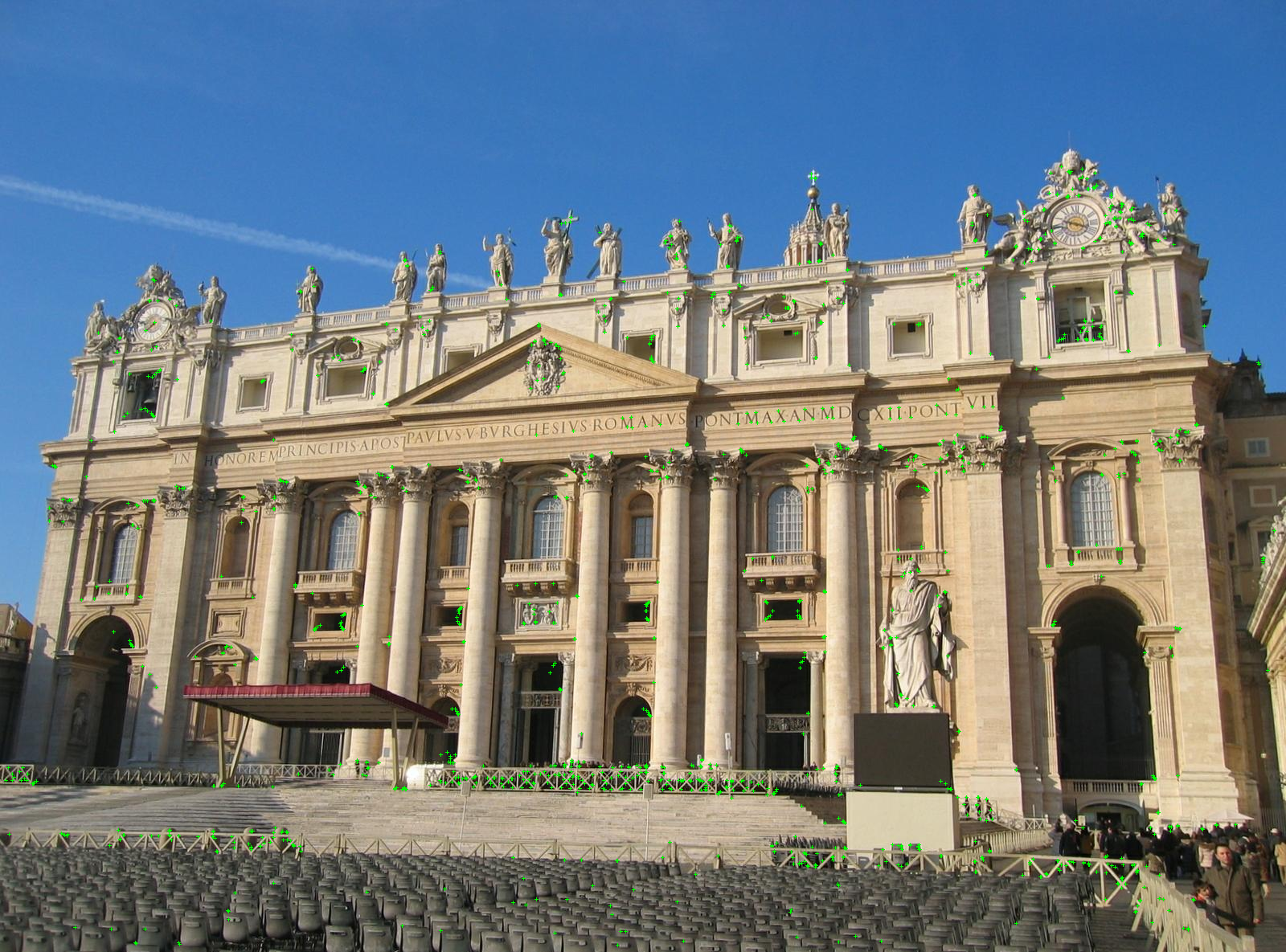}
    \caption{\textbf{Qualitative example of REKD (which uses a rotation equivariant architecture) keypoint detections.} In contrast to \ours, while there are keypoints near the cross, there are no keypoints on the cross. If one observes the figure closely, it can be observed that there in fact no keypoint on light colored pixels.}
    \label{fig:qual-rekd}
\end{figure*}

\begin{figure*}
    \centering
    \includegraphics[width=\linewidth]{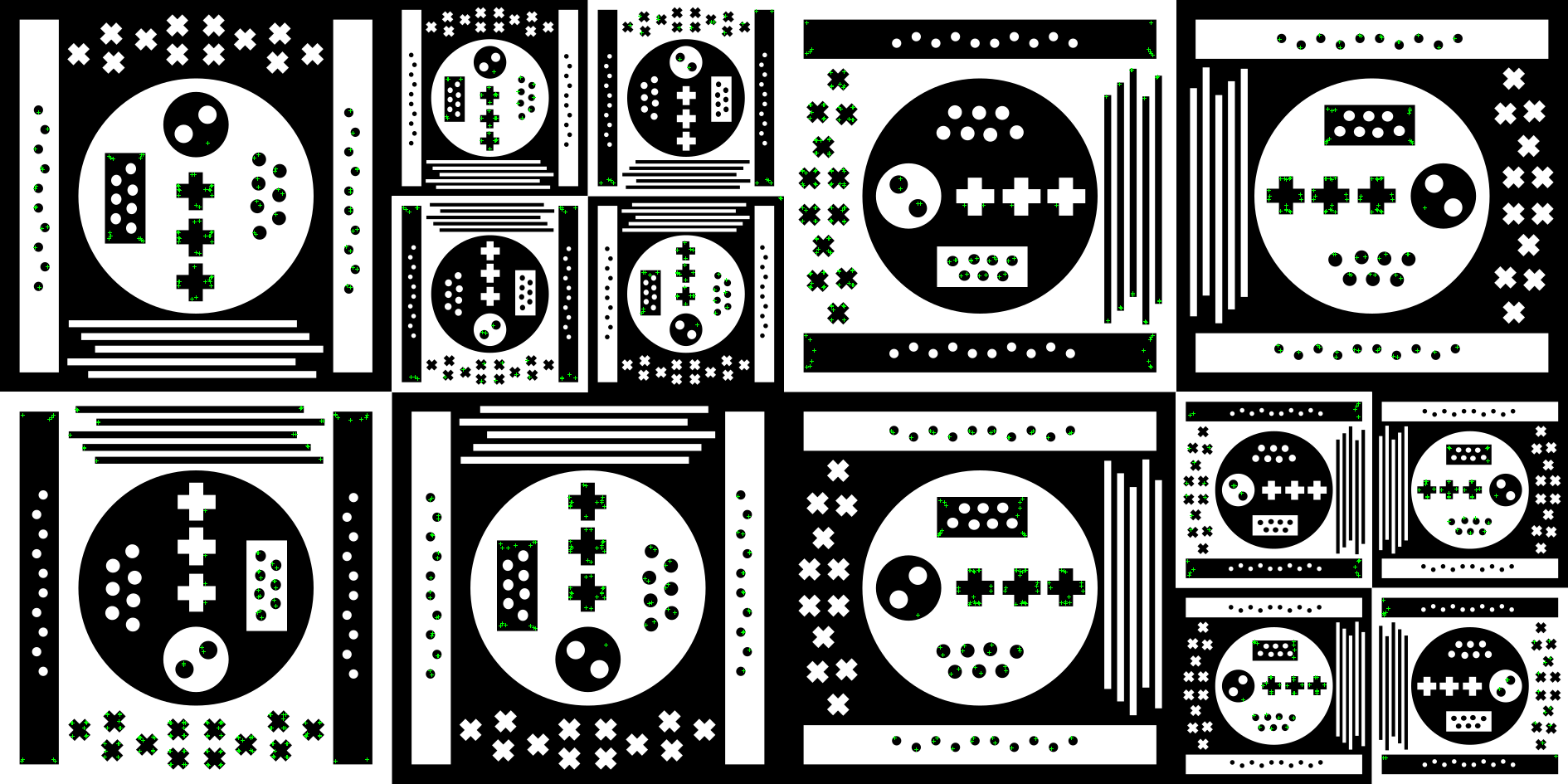}
    \caption{\textbf{Detections of the rotation equivariant detector REKD~\citep{lee2022self}.} Again, we observe the same bias toward detecting only dark keypoints as described in~\Cref{sec:emerge}.}
    \label{fig:rekd-dets}
\end{figure*}

\section{Light Dark Equivariance Through Augmentation?}
\label{suppl:failed-aug}
As described in the main text, and in the previous section. Light and dark keypoint detectors tend to emerge from training.
We attempted to remedy this by augmentation.
To this end we always negated one of either $\mathbf{I}^A$ or $\mathbf{I}^B$.
By negation we mean both the operation $\mathbf{I}_{\text{neg}} = 1 - \mathbf{I}$ in the original RGB colorspace, and luminance negation which is done by converting the image to the HSL colorspace, inverting the luminance channel, and then converting back to RGB. 
Note that in the case of grayscale images these operations are equivalent.

While we found that this leads to detectors that detect both light and dark keypoints, it tends to begin by learning \emph{line-like} detections, as those are invariant to the brightness.
We find that during learning the model learns to discard most of these (retaining junctions), but has difficulty finding keypoints not on any line, and thus typically performs worse than the light or dark counterparts.

An alternative approach to augmentation would be to explicitly design the network to be equivariant to the action of the image negation group. 
Color equivariance has previously been explored~\citep{lengyel2024color,o'mahony2025learning}, but to the best of our knowledge not for image negations. 
In the simplest case, one could include both the image and its' negation, process them independently, and merge the predictions with the max operation, as done in the main paper. 
This has the downside of requiring twice the compute, and it is furthermore not obvious that detectors trained in this way would perform well, as negated images are qualitatively different from regular images.
Nevertheless, we believe that an equivariant architecture would likely work well, as it has previously been demonstrated to work for rotations~\citep{lee2022self,santellani2023strek}.

\section{Kiki vs Bouba Detectors}
\label{suppl:bouba-kiki}
In~\Cref{fig:bouba-kiki} we qualitatively illustrate the difference between bouba and kiki detectors.
Kiki detectors only detect keypoints near edges, while bouba detectors tend to prefer centers of mass of blobs. 
For example, for the crosses in~\Cref{fig:bouba-kiki}, the bouba light detector detects the center point, while the kiki dark detector detects all corners.
While both bouba and kiki detectors are common for dark detectors, we did not commonly observe kiki light detectors. 
Hence our usage of a bouba light detector for \ours.
\begin{figure*}
    \centering
    \includegraphics[trim=0cm 0cm 50cm 0cm, clip,width=.495\linewidth]{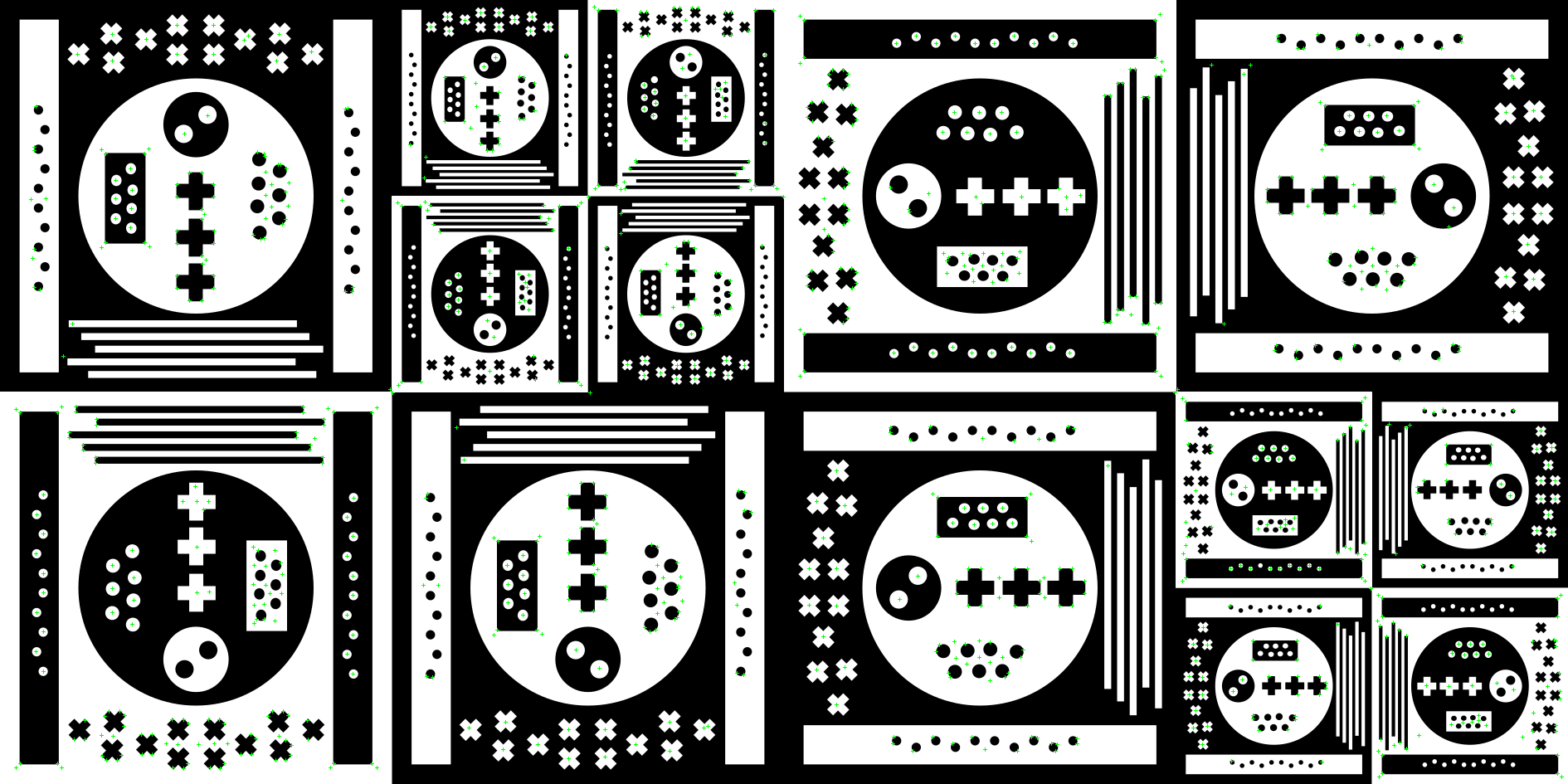}
    \includegraphics[trim=0cm 0cm 27.09cm 0cm, clip,width=.495\linewidth]{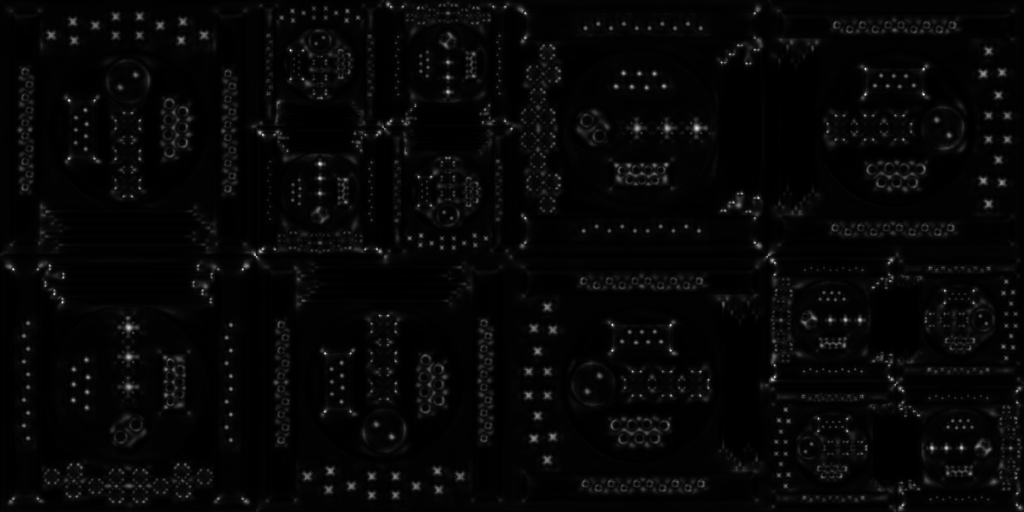}
    \caption{\textbf{Bouba light detector --- Kiki dark detector.} \textbf{Left:} Keypoints from~\ours, which is a combination of a bouba light detector and a kiki dark detector. Note in particular that keypoints on the white crosses are centered, while the keypoints on the dark crosses are on the edges. \textbf{Right:} The corresponding scoremaps. Here we also see that the scoremap of bouba detectors is in general more soft than kiki detectors.}
    \label{fig:bouba-kiki}
\end{figure*}

\section{Differentiable RANSAC vs Policy Gradient vs Reinforcement Learning}
\label{suppl:diff-ransac}
The policy gradient formulation used in DISK~\citep{tyszkiewicz2020disk} and \ours~has similarities with differentiable approximations used for RANSAC. 
We thus recap that related work here.

DSAC~\citep{brachmann2019neural} and NG-RANSAC~\citep{brachmann2017dsac} estimate gradients of the pose error by a variance reduced~\citep{sutton2018reinforcement} REINFORCE~\citep{williams1992simple} estimator, which reads
\begin{equation}
    \nabla_{\theta} \mathcal{L} \approx \sum_{i=1}^H (\epsilon_{\text{pose}}(\mathcal{H}(\mathcal{K}_i)-\mathbb{E}[\epsilon_{\text{pose}}])\nabla_{\theta}\log p_{\theta}(\mathcal{K}_i),
\end{equation}
where
\begin{equation}
    \mathcal{L} = \mathbb{E}[\epsilon_{\text{pose}}(\mathcal{K})].
\end{equation}

\citet{wei2023generalized} use a straight-through Gumbel-softmax estimator~\citep{jang2017categorical} to estimate the gradient of the keypoint positions with respect to keypoint scores. 
Relatedly, Key.Net~\cite{barroso2019key} and ALIKED~\citep{Zhao2023ALIKED} uses a local softmax operator to sample keypoints, in order to make their repeatability objective differentiable w.r.t the scoremap.
In this work we focus on the scoremap perspective, without directly differentiating the keypoint positions.

These approach is related to ours and DISK~\citep{tyszkiewicz2020disk} through the usage of the identity $\nabla_{\theta} \log p(\theta) = \frac{\nabla_{\theta}p(\theta)}{p(\theta)}$, which motivates REINFORCE and policy gradient.

We do believe that differentiable RANSAC has potential for self-supervised learning of keypoints. 
However, we found that the gradients, at least in our implementations, were in general not reliable as a supervision signal.

\section{Pointwise Maximum of Two Distributions}
\label{suppl:max}
For our distillation target we use the point-wise maximum of the light and dark keypoint detector.
It turns out that this way of merging distributions has some nice properties.
In particular, here we show that, under some mild assumptions, this merged distributions' local maxima is the union of the local maxima of the two distributions.

We'll first define what we mean by a keypoint.
\begin{definition}[Keypoint]
We define a \emph{keypoint} $f \ge 0$ as a $\mathcal{C}^2$, unimodal function with a compact support around its mode $\mu_f := \text{argmax}_x f(x)$.
\end{definition}

When we ensemble two detectors, it may be the case that keypoints ``disappear''. We will call such keypoints \emph{subsumed}. 
\begin{definition}[Subsumed keypoint]
We say that a keypoint $f\ge0$ is \emph{subsumed} by another keypoint $g\ge0$ if there is no neighbourhood $\mathcal{X}, \mu_f\in\mathcal{X}$ such that $f(\mu_f) \ge g(\mu_f)$. This is visually illustrated in~\Cref{fig:subsumed}.
\end{definition}
\begin{figure}
    \centering
    \includegraphics[width=\linewidth]{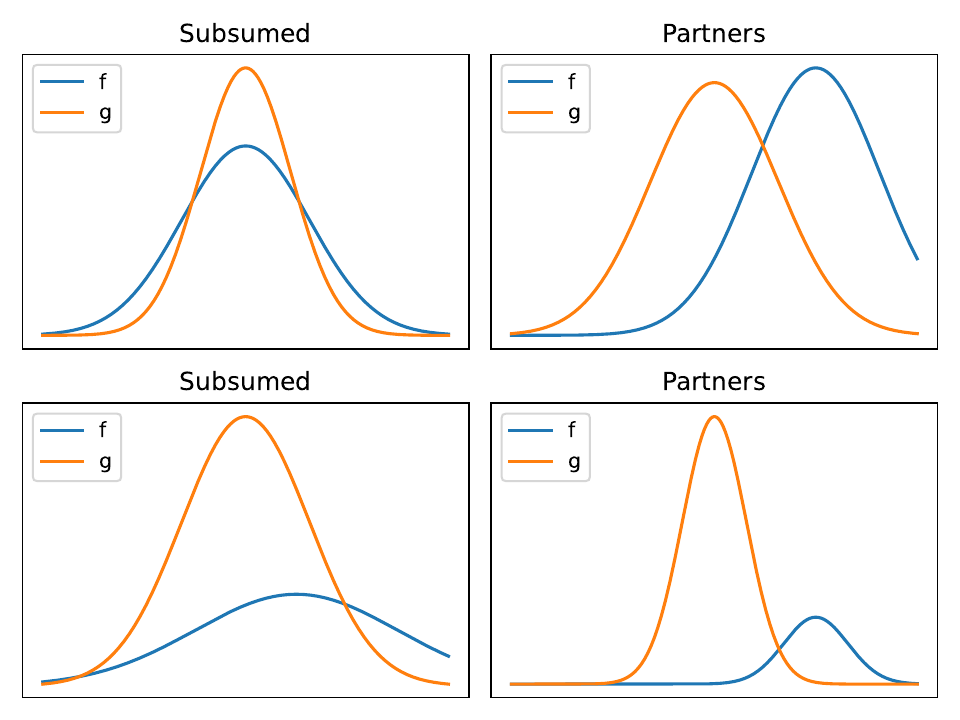}
    \caption{\textbf{Illustration of subsumed and partner keypoints.} Taking the maximum of partner keypoints preserve the local maxima.}
    \label{fig:subsumed}
\end{figure}
In contrast, when neither of two keypoints subsume the other, we'll call them \emph{partners}.
\begin{definition}[Partner keypoints]
\label{def:partner-kps}
    Let $f,g$ be keypoints such that neither of them subsumes the other. We call such a pair \emph{partners}.
\end{definition}
Next we expand this to cover distributions.
\begin{definition}[Keypoint distribution]
A keypoint distribution is a distribution that factors as a sum of keypoints with disjoint support.
\begin{equation}
    p(x) = \sum_{k=1}^K f_k(x)
\end{equation}
\end{definition}
\begin{theorem}
    The set of local maxima of a keypoint distribution $p$ is 
    \begin{equation}
        \mu_p = \{\mu_{f_k}\}_{k=1}^K.
    \end{equation}
\end{theorem}
\begin{proof}
    This follows immediately from the fact that $f_k$ have disjoint support, hence $\forall x$ $p(x) = f_k$ for some $k$.
\end{proof}

\begin{definition}[Partner distributions]
Let $p$ and $q$ be keypoint distributions, and let all $f$ in $p$ and $g$ in $q$ be partners. We then call $p$ and $q$ \emph{partners}.
\end{definition}

One can wonder whether we might get additional local maxima from the pointwise-max operation.
This is not the case, as the following simple theorem shows.
\begin{theorem}[No extra maxima]
\label{thm:no-more-max}
    Let $p,q$ be $\mathcal{C}^2$.
    Then $m(x) =\max(p(x),q(x))$ does not contain any local maxima that are not in $\mu_p \cup \mu_q$.
\end{theorem}
\begin{proof}
    Assume there is a local maximum $y$, $y \notin \mu_p, y\notin \mu_q$, of $m$.
    If $p(y) > q(y)$ then there exists a local neighbourhood $\mathcal{Y}$ of $y$ such that $m(y)=p(y),\forall y\in \mathcal{Y}$ . 
    However, this would imply that $y$ is a local maximum of $p$, which is a contradiction.
    The same argument applies for $q(y) > p(y)$.
    If no such neighbourhood exists, then $m(y)=p(y)=q(y)$ and such points cannot be local maxima as $p(y) \leq m(y)$ and if $y$ is not a local maximum of $p$ it cannot be a local maximum of $m$.
\end{proof}

Next we use the above to show that the pointwise max of partner keypoints have exactly the union of their original local maxima. 
\begin{theorem}
    Let $f,g$ be partner keypoints.
    Then the set of local maxima of $m(x) =\max(f(x),g(x))$ is $\{\mu_f, \mu_g\}$.
\end{theorem}
\begin{proof}
    Let $x = \mu_f$. As $f$ is not subsumed by $g$ there exists some neighbourhood $\mathcal{X}$ around $x$ where $f\geq g$, hence $m(x) = f(x) ,\forall x \in \mathcal{X}$, as $x$ is a local maximum of $f$ it is also a local maximum of $m$. 
    The same argument holds in the other direction, hence $\mu_g$ is also a local maximum of $m$.
    By~\Cref{thm:no-more-max} there are no other maxima.
    Hence the local maxima of $m$ are $\{\mu_f, \mu_g\}$.
\end{proof}

Now we put together the pieces, and show that partner distributions have exactly the union of the original keypoint distributions' local maxima.

\begin{theorem}[Max retains local maxima]
Let $p$ and $q$ be partners.
Then the set of local maxima of $\max(p,q)$ is the same as the union of the local maxima of $p$ and $q$ seperately.
\end{theorem}
\begin{proof}
    For partner distributions, all keypoints are partners. 
    As $f_k$ have disjoint support we can write the local maxima as
    \begin{equation}
        \cup_k (\{\mu_{f_k}\} \cup \{\mu_{g_j}\}_j) = \{\mu_{f_k}\}_k \cup \{\mu_{g_j}\}_j.
    \end{equation}
    Which is what we wanted to show.
\end{proof}

\section{Further Architectural Details}
\label{suppl:arch}
For our all our experiments we use the DeDoDe-S model.
This model has an encoder-decoder architecture.
The encoder is a VGG11~\citep{simonyan2015vgg} network, from which the layers corresponding to strides $\{1,2,4,8\}$ are kept. 
These have dimensionality $\{64, 128, 256, 512\}$ respectively.
The decoder at each stride consists of 3 [5x5 Depthwise Convolution, BatchNorm, ReLU, 1x1 Convolution] blocks. 
The output of each such block is $\{1, 32+1, 128+1, 256+1\}$ dimensional, and contains context that is upsampled, as well as keypoint distribution logits. 
The input at each stride is the output of the corresponding encoder, and for stride $\{1,2,4\}$ additionally also the outputs of the previous decoder layer. The input dimensionality for each stride is thus $\{64+32, 128+128, 256+256, 512\}$. 
The internal dimensionality of the decoders is $\{32, 64, 256, 512\}$.

\section{Runtime Comparisons}
\label{suppl:runtime}
In \Cref{tab:runtime} we compare the runtime of different detectors on a A100 for a batchsize of 1 (a single image). 
We use the same inference settings as in our SotA experiments (see\Cref{sec:inference-settings}, and \Cref{suppl:inference} for details).
As can be seen in the figure \ours~has a similar runtime as previous detectors.

\begin{table}[]
    \centering
    \caption{\textbf{Runtime comparison.} We measure the runtime for a batchsize of 1 on a A100 GPU for detectors using the same inference settings as in the SotA comparisons. Measured in milliseconds (lower is better).}
    \begin{tabular}{l r}
        \toprule
        Method & Runtime (ms) \\
        \midrule
         ALIKED & 8.89\\
         REKD & 1499.78 \\
         DISK & 14.30\\
         SuperPoint & 6.46\\
         DeDoDe-v2 & 42.76\\
         \ours & 18.74\\
         \bottomrule
    \end{tabular}
    \label{tab:runtime}
\end{table}
\section{Inference Settings}
\label{suppl:inference}
\paragraph{DeDoDe v2:} We use the same settings as recommended by the authors, using a $784\times 784$ resolution, with a NMS radius of $3\times 3$.
\paragraph{ALIKED, SuperPoint, DISK:} We use the same settings as in LightGlue~\citep{lindenberger2023lightglue} and resize the longer side of the image to 1024. For ALIKED a 5x5 NMS, SuperPoint a 4x4 NMS, DISK a 5x5 NMS, and for SIFT no additional NMS is used.
\paragraph{REKD:} We use the original image size, as recommended by the authors.

\section{Evaluation Protocol}
\label{suppl:eval-protocol}
\paragraph{Pose Estimation:}
We follow LightGlue~\citep{lindenberger2023lightglue} and use PoseLib (version 2.0.4 via pip) for pose estimation. We use the default settings of PoseLib for RANSAC, and set the threshold to $2$ pixels for all benchmarks. To reduce noise, we additionally run all RANSAC 5 times, permuting the correspondences to induce randomness in PoseLib.

\paragraph{Evaluation Metrics:}
For Fundamental matrix and Essential matrix estimation we report the AUC\@$5^\circ$ of the pose error, where the pose error 
\begin{equation}
    \epsilon_{\rm pose}((\hat{\mathbf{q}},\hat{\mathbf{t}}), (\mathbf{q}_{\rm GT},\mathbf{t}_{\rm GT}))= \max(\angle(\hat{\mathbf{t}}, \mathbf{t}_{\rm GT}), \angle(\hat{\mathbf{q}},\mathbf{q}_{\rm GT})),
\end{equation}
i.e., the maximum of the rotational angle and the translational angle. Note that both the rotational and the translational error are dimensionless as two-view geometry is only defined up-to-scale.
For Essential matrix estimation the estimated pose $(\hat{\mathbf{q}},\hat{\mathbf{t}})$ is retrieved by disambiguation of the 4 possible relative poses using that $\mathbf{E} = [\mathbf{t}]_{\times}\mathbf{R}$. For the Fundamental matrix the process is the same, where the Essential matrix is computed by decomposing $\hat{\mathbf{E}} = \mathbf{K}^{-1} \hat{\mathbf{F}}\mathbf{K}$.

For Homography Estimation we report the AUC\@3 pixels of the corner end-point-error, which is defined as
\begin{equation}
    \epsilon_{\rm epe}(\hat{\mathbf{H}}, \mathbf{H}_{\rm GT}) = \frac{1}{4}\sum_{\mathbf{x}\in\text{corners}} \lVert \pi(\hat{\mathbf{H}}[\mathbf{x};1])- \pi(\mathbf{H}_{\rm GT}[\mathbf{x};1])\lVert_2,
\end{equation}
where $\pi: [x;y;z] \to [x/z;y/z]$ is the projection operator. We additionally follow previous work~\citep{sun2021loftr, edstedt2024roma} and normalize the epe error by $480/\min(H,W)$.

\begin{figure*}
    \centering
    \includegraphics[width=\linewidth]{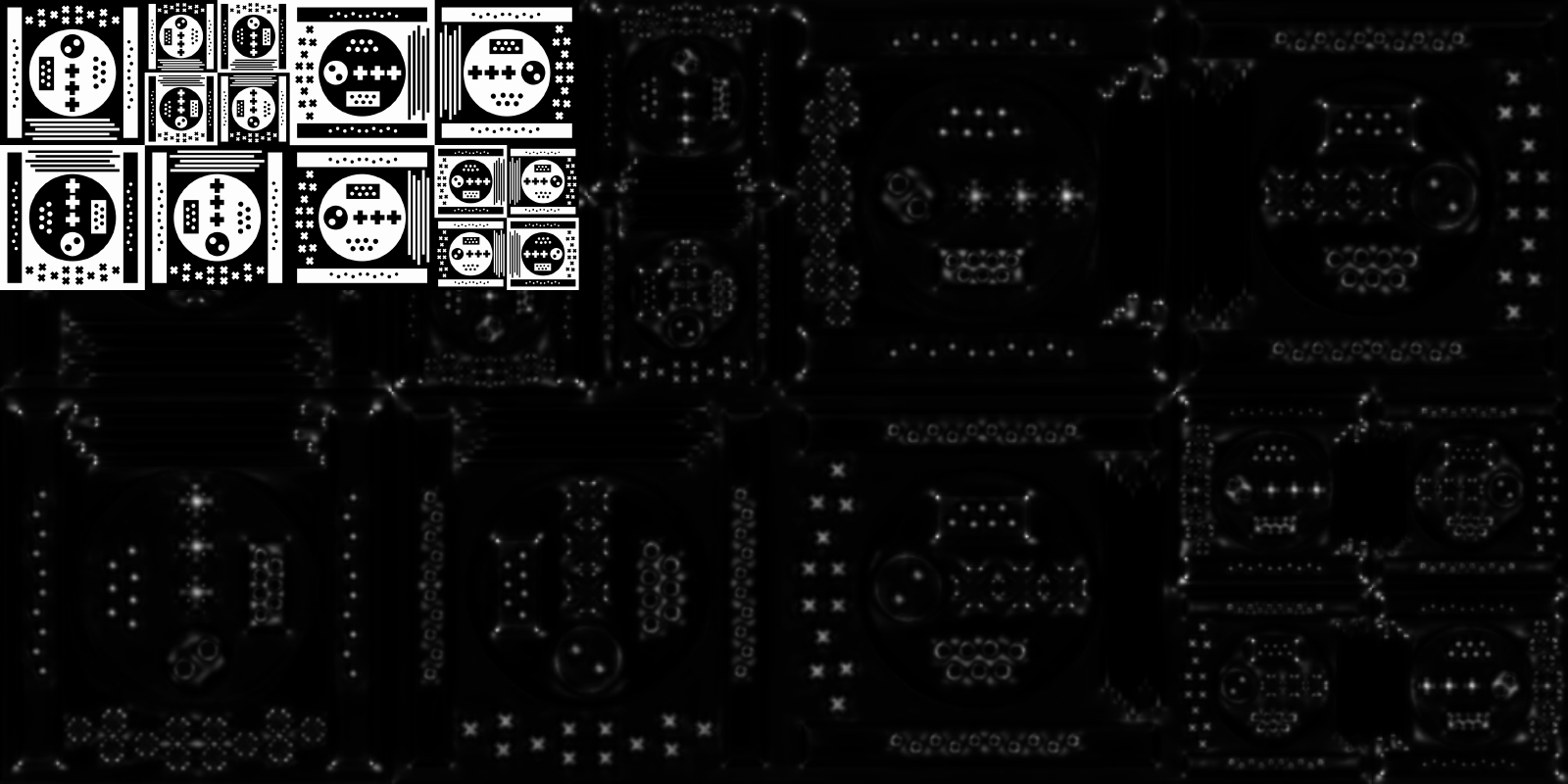}
    \caption{
    \textbf{Scoremap distribution $p_\theta$ of \ours.} 
    \emph{Overlay upper left:} Input image. \emph{Background:} Learned scoremap distribution $p_{\theta}$ of \ours. \ours is distilled from two self-supervised keypoint detectors (\Cref{sec:keypoint}) trained via reinforcement learning (\Cref{sec:rl}) to maximize per-image reward (\Cref{sec:reward}) using off-policy top-k sampling at local maxima (\Cref{sec:sampling}), with some regularization (\Cref{sec:regularization}), resulting in a simple objective (\Cref{sec:objective}). Intriguingly, several types (Light/Dark, Bouba/Kiki) of detectors emerge from this objective (\Cref{sec:emerge}). We find we can harness these into a single, more diverse detector, through knowledge distillation (\Cref{sec:distill}).
    }
    \label{fig:scoremap}
\end{figure*}
\section{Further Qualitative Results}
We show a qualitative example of \ours~scoremap in~\Cref{fig:scoremap}.

\end{document}